\documentclass[10pt,twocolumn,letterpaper]{article}

\usepackage{cvpr}              

\usepackage[dvipsnames]{xcolor}

\definecolor{cvprblue}{rgb}{0.21,0.49,0.74}
\usepackage[pagebackref,breaklinks,colorlinks,citecolor=cvprblue]{hyperref}
\usepackage{multirow}
\usepackage{cite}
\usepackage{amsmath,amssymb,amsfonts}
\usepackage{algorithmic}
\usepackage{graphicx}
\usepackage{textcomp}
\usepackage{xcolor}
\usepackage{booktabs}
\usepackage{pgfplots}
\usepackage{subfloat}
\usepackage{amsthm}
\usepackage{rotating}
\usepackage[accsupp]{axessibility} 
\def\sysname{Multi-MaP}
\newtheorem{thm}{Theorem}
\newtheorem{cor}[thm]{Corollary}
\def\paperID{10570} 
\def\confName{CVPR}
\def\confYear{2024}

\title{Multi-Modal Proxy Learning Towards Personalized Visual Multiple Clustering}

\author{Jiawei Yao$^1$ \quad Qi Qian$^{2}$ \quad Juhua Hu$^1$\thanks{Corresponding author}\\
$^1$ School of Engineering and Technology, University of Washington, Tacoma, WA 98402, USA\\
$^2$ Alibaba Group, Bellevue, WA 98004, USA\\
{\tt\small \{jwyao, juhuah\}@uw.edu, qi.qian@alibaba-inc.com}
}

\begin{document}
\maketitle

\begin{abstract}
Multiple clustering has gained significant attention in recent years due to its potential to reveal multiple hidden structures of data from different perspectives. The advent of deep multiple clustering techniques has notably advanced the performance by uncovering complex patterns and relationships within large datasets. However, a major challenge arises as users often do not need all the clusterings that algorithms generate, and figuring out the one needed requires a substantial understanding of each clustering result. Traditionally, aligning a user's brief keyword of interest with the corresponding vision components was challenging, but the emergence of multi-modal and large language models (LLMs) has begun to bridge this gap. In response, given unlabeled target visual data, we propose \sysname{}, a novel method employing a multi-modal proxy learning process. It leverages CLIP encoders to extract coherent text and image embeddings, with GPT-4 integrating users' interests to formulate effective textual contexts. Moreover, reference word constraint and concept-level constraint are designed to learn the optimal text proxy according to the user’s interest. \sysname{} not only adeptly captures a user's interest via a keyword but also facilitates identifying relevant clusterings. Our extensive experiments show that \sysname{} consistently outperforms state-of-the-art methods in all benchmark multi-clustering vision tasks. Our code is available at \href{https://github.com/Alexander-Yao/Multi-MaP}{https://github.com/Alexander-Yao/Multi-MaP}.

\end{abstract}    
\vspace{-0.3cm}
\section{Introduction}
\label{sec:intro}

Clustering, which groups data points based on their similarities, has been extensively researched, since huge amount of unlabeled data are becoming more and more available.
Traditional methods~\cite{macqueen1967some, ng2001spectral, bishop2006pattern} exploit general-purpose handcrafted features that are not always ideal for specific tasks. Recently, deep clustering algorithms~\cite{xie2016unsupervised,guerin2018improving, qian2022unsupervised} leverage Deep Neural Networks (DNNs) significantly improve performance. However, most algorithms yield a single data partition, while data can exhibit multiple aspects (e.g., color and species as shown in Fig.~\ref{fig:intro1}). Traditional multiple clustering algorithms~\cite{bae2006coala, hu2017finding} have been developed to generate different partitions for varying applications, demonstrating the ability to identify multiple distinct clusterings from a dataset. Contemporary advancements in the field reveal a growing inclination among researchers to integrate deep learning methodologies for facilitating multiple clustering outcomes. Predominantly, such techniques capitalize on auto-encoders and data augmentation processes to capture a broad spectrum of feature dimensions, thereby enhancing the performance of multiple clustering~\cite{miklautz2020deep,ren2022diversified,DBLP:conf/inns-dlia/YaoLRH23}.

\begin{figure}[t]
    \centering
    \includegraphics[width=\columnwidth]{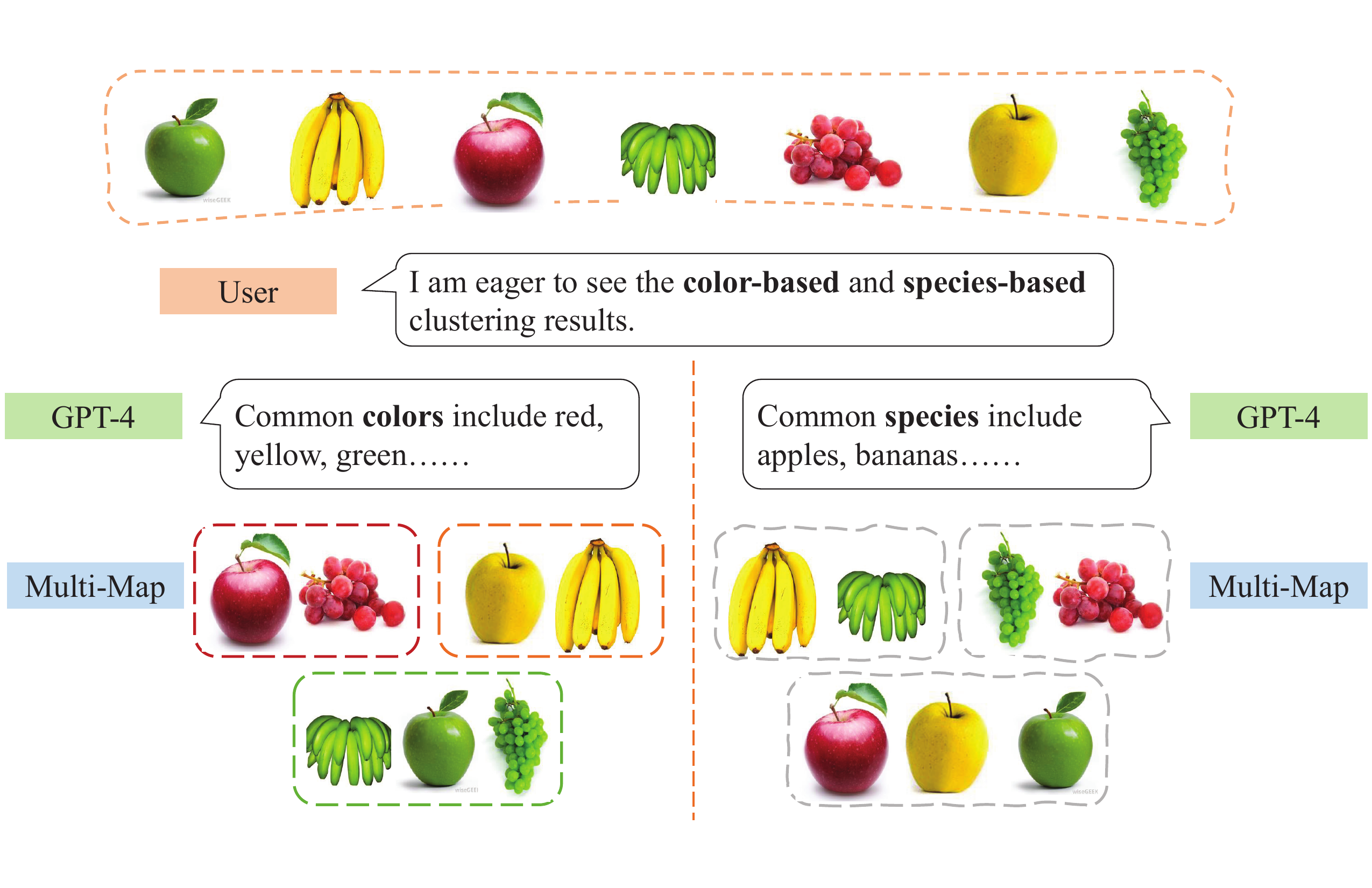}
    \vspace{-1.2cm}
    \caption{The flow chart of \sysname{}. \sysname{} obtains multiple clustering results based on the high-level concepts from users and the reference words from GPT-4.}
    \label{fig:intro1}
    \vspace{-0.5cm}
\end{figure}

However, a common issue arises as users often do not require all the clusterings generated by the algorithm, and identifying the relevant ones necessitates a substantial understanding of each clustering result. Therefore, in this work, we initiate an exploration into a method that is adept at accurately capturing and reflecting a user's interest. Users typically express their interests through concise keywords (e.g., color or species), and aligning these with different visual components precisely is challenging. Fortunately, the advent of multi-modal models like CLIP~\cite{radford2021learning} that aligns images to their corresponding text descriptions, can be helpful to fill this gap. However, unlike methods that employ labeled data to fine-tune pre-trained models~\cite{gao2023clip,wang2023improved}, multiple clustering deals with environments marked by vague or undefined label categories and amounts. Consequently, given only a high-level concept from the user, it is infeasible to fine-tune the pre-trained models to capture a specific aspect of the data, without the detailed labels corresponding to the user's concept.

An intuitive strategy to integrate pre-trained models into clustering is the zero-shot feature extraction, followed by clustering of the resultant embeddings. However, this approach exhibits limitations, particularly in capturing the interests of users within the dataset. Taking the multi-modal model CLIP~\cite{radford2021learning} as an example, when feeding image data into CLIP, regardless of what aspects of clustering the user expects, CLIP can only produce the same embeddings. Even considering the scenario that different pre-trained models can capture different aspects of the same data as in ~\cite{guerin2018improving}, it is hard to tell which one matches a user's preference. Fortunately, given CLIP's ability to model image-text pairs collaboratively, we can use a user's high-level concept to trigger the corresponding feature extraction from the pre-trained encoders from CLIP. However, no previous work has studied if CLIP has this potential to uncover different aspects of images, which is the focus of this work.

Specifically, we propose to integrate a user's high-level concept describing the preference using a personalized textual prompt. For example, if a user's focus pertains to the color dimension of fruit, a prospective prompt might be formulated as ``a fruit with the color of *'', wherein the ``*'' placeholder represents the proxy word awaiting determination using the knowledge in CLIP. Thereafter, we can learn the proxy word embedding by maximizing the similarity between the image and text embedding. However, the proxy word embedding is now searched in a continuous space while the original CLIP used discrete tokens, which can downgrade the performance. We prove that the performance can be well guaranteed by selecting the nearest token as the reference, which is unavailable in a clustering task.

Fortunately, we can use the user's high-level concept as a reference, which however covers a broad range of tokens under its scope. Therefore, we propose to leverage large language models like GPT-4 to generate candidate tokens using the user's high-level concept, in which the closest can be used as a closer reference token. Furthermore, in some scenarios, users may provide multiple concepts to obtain multiple clusterings simultaneously as shown in Fig.~\ref{fig:intro1}, we can also introduce a negative loss with these constrastive concepts to further enhance the learning. Therefore, to capture a user's specific interest and discover personalized clustering structure hidden in the data, we propose a multi-modal proxy learning method (\sysname{}). \sysname{} incorporates both text prompts and unlabeled images from the clustering task, and leverages CLIP to acquire their respective personalized representations using both reference word and concept-level constraints.~The contributions of this work can be summarized as:
\begin{itemize}
    \item We are the first to explore a deep multiple clustering method that precisely captures a user's interest(s) and generates personalized clustering(s) accordingly. 
    
    \item We propose a novel multi-modal proxy learning method, \sysname{}, by leveraging the text and image encoders pre-trained by CLIP, where the user's interest can be captured by the personalized text prompts. 
    
    \item Considering the challenge of learning a word proxy in a continuous space while tokens in CLIP were discrete, we theoretically prove that a close reference token can help constrain the search, which motivates the proposed reference word constraint and concept-level constraint. 
 
    \item We conduct extensive experiments on all publicly available visual multiple clustering tasks, which empirically demonstrates the superiority of the proposed \sysname{}, with a precise capturing of a user's interest.

    \item Finally, to the best of our knowledge, we are the first who demonstrate that CLIP can uncover different semantic aspects of images.
\end{itemize}
\section{Related Work}
\label{sec:formatting}
In this section, we briefly review two related directions, that is, multiple clustering and multi-modal models.

\subsection{Multiple Clustering}
Multiple clustering, as a kind of method that can discover alternative perspectives of data, has attracted considerable attention.
Traditional methods for multiple clustering~\cite{hu2018subspace} use shallow models to find different ways of grouping data. Some of these methods rely on constraints to produce alternative clusterings. For instance, COALA~\cite{bae2006coala} uses the objects in an existing clustering as constraints for creating a new clustering and Qi et al.~\cite{qi2009principled} formulated multiple clustering as a constrained optimization problem. 
Other methods exploit different feature subspaces to generate multiple clusterings. For example, Hu et al.~\cite{hu2017finding} found multiple clusterings by maximizing the eigengap in different subspaces. MNMF~\cite{yang2017non} adds the inner product of similarity matrices as a regularization term to find multiple clusterings. Some methods also use information theory to generate multiple clusterings. Gondek~\cite{gondek2003conditional} applied conditional information bottleneck and
Dang~\cite{dang2010generation} used an expectation maximization framework to optimize mutual information. 

Recently, some methods have used DNNs to find multiple clusterings and achieved better results. Wei et al.~\cite{wei2020multi} proposed a deep matrix factorization based method that uses multi-view data to find multiple clusterings. ENRC~\cite{miklautz2020deep} uses an auto-encoder to learn object features and finds multiple clusterings by optimizing a clustering objective function. iMClusts~\cite{ren2022diversified} exploits auto-encoders and multi-head attention to learn features from different perspectives, and then find multiple clusterings. AugDMC~\cite{DBLP:conf/inns-dlia/YaoLRH23} leverages data augmentation to generate diverse aspects of images and learns the representations to discover multiple clusterings. Although existing deep multiple clustering methods have achieved remarkable results, they require users to exert considerable efforts to select the correct clustering they need. In this work, we aim to efficiently and effectively capture users' interests using short keywords and provide clustering results accordingly.

\vspace{-0.1cm}

\subsection{Multi-modal Model}
\vspace{-0.1cm}
Multi-modal learning refers to the process of learning representations from different types of input modalities, such as image data, text, or speech. As related, we focus on how vision models can benefit from natural language supervision. 
CLIP~\cite{radford2021learning} is a notable model, which is trained with a dataset containing 400 million text-image pairs from the internet. The objective is to align images to their corresponding text using contrastive learning. 

Fine-tuning approaches adapt vision-language models like CLIP to specific downstream image recognition tasks. CoOp~\cite{zhou2022learning} and CLIP-Adapter~\cite{gao2023clip} exemplify this, with the latter integrating residual style feature blending to enhance performance on various visual classification tasks. Additionally, insights from TeS~\cite{wang2023improved}, further elucidate the effectiveness of fine-tuning strategies in leveraging natural language supervision for enhanced visual understanding. Recognizing the scarcity of labeled data for various tasks, significant research efforts have been dedicated to enhancing zero-shot learning. Some approaches extend beyond CLIP by incorporating other large pre-trained models. For instance, VisDesc~\cite{menon2022visual} harnesses the power of GPT-3 to generate comprehensive contextual descriptions corresponding to given class names, thereby demonstrating superior performance compared to CLIP's basic prompts. UPL~\cite{huang2022unsupervised} and TPT~\cite{shu2022test} leverages unlabeled data to optimize learnable input text prompts. InMaP~\cite{inmap} recovers the proxy of each class in the vision space with the help from the text proxy. All of these methods aim to improve the performance of vision classification tasks, while clustering is a different scenario that we do not have class names that we can exploit to extract useful information from CLIP. In this work, we consider leveraging the CLIP encoder to extract coherent text and image embeddings, and using GPT-4 to integrate users' interests for the purpose of multiple clustering.
\vspace{-0.1cm}

\vspace{-0.1cm}
\section{The Proposed Method}
\vspace{-0.2cm}

In this section, we first briefly review the training objective in CLIP as follows, and then describe the details of our proposed method based on that.
\begin{figure*}[t]
    \centering
    \includegraphics[width=0.9\textwidth]{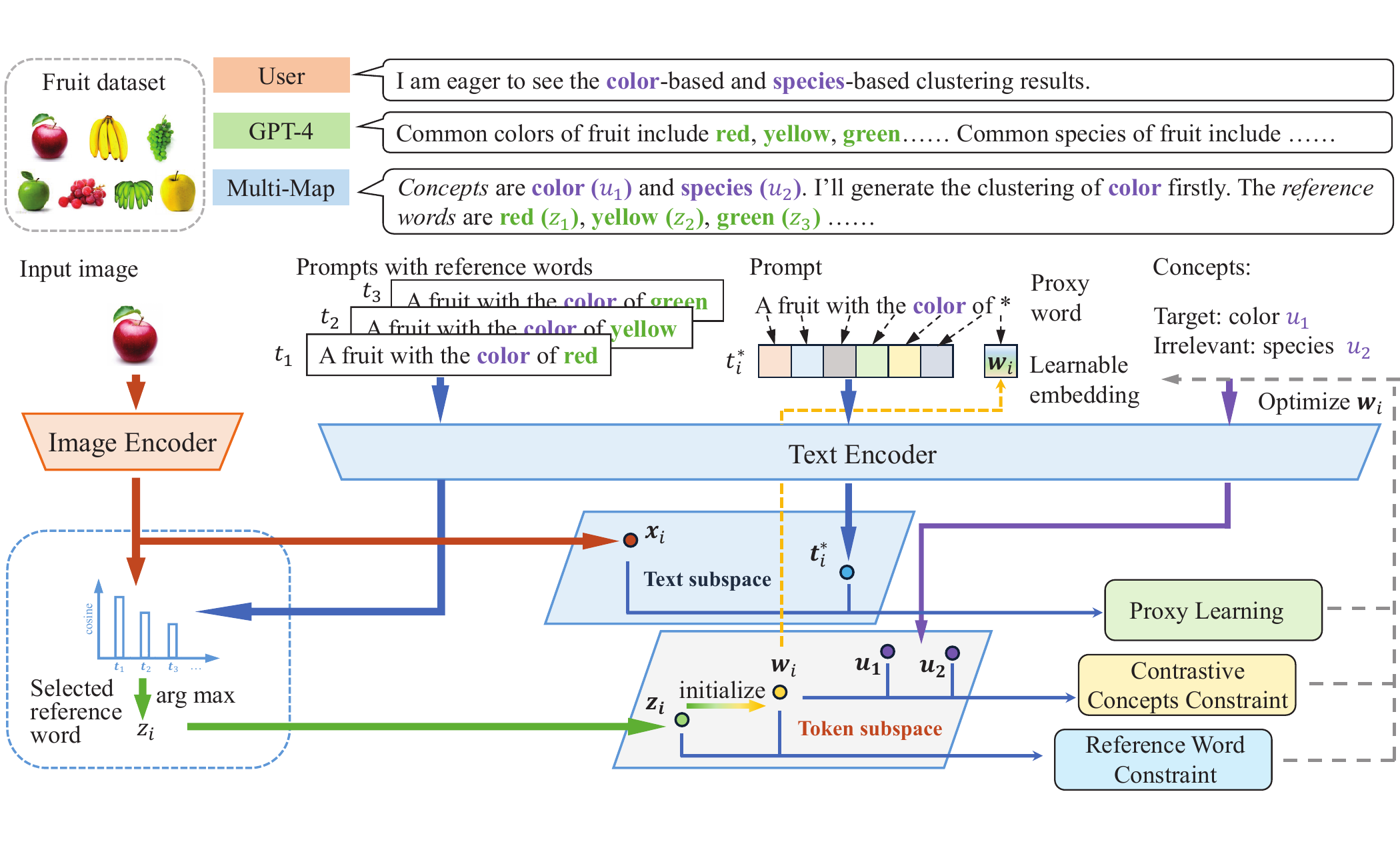}
    \vspace{-0.9cm}
    \caption{\textbf{\sysname{} framework.} In the training process of \sysname{}, the vision and text encoders are frozen and the proxy word embeddings $\mathbf{w}_i$ are learnable. Specifically, it first constructs the prompt embeddings based on the reference words provided by GPT-4 using a user's high-level concept, and then selects a reference word $z_i$ for each image according to the similarity between the prompt embeddings $\mathbf{t}_i$ and the image embeddings $\mathbf{x}_i$. Then, it combines the prompt and the reference words to form the new prompt embeddings $\mathbf{t}_i^*$ and maximizes the similarity to the image representation, so the proxy word embeddings $\mathbf{w}_i$ can capture the desired image features. In addition, the proxy word embeddings $\mathbf{w}_i$ should be close to the target concept word $\mathbf{u_1}$ and the selected reference word $\mathbf{z}_i$ to construct the concept-level constraint and reference word constraint, which capture the features related to the user's interest.}
    \label{fig:intro}
    \vspace{-0.5cm}
\end{figure*}

\subsection{Multi-modal Pre-training}

Given a set of image-text pairs as $\{x_i, t_i\}_{i=1}^n$, where $x_i$ is an image and $t_i$ is the corresponding text description, their vision and text representations can be obtained by two encoders as $\mathbf{x}_i = f(x_i)$ and $\mathbf{t}_i = h(t_i)$. $f(\cdot)$ and $h(\cdot)$ are vision and text encoders for optimization, where $\mathbf{x}_i$ and $\mathbf{t}_i$ have the unit norm. Then, these two encoders can be learned by minimizing the contrastive loss as
\[\min_{f,h} \sum_i -\log\frac{\exp(\mathbf{x}_i^\top \mathbf{t}_i/\tau)}{\sum_j \exp(\mathbf{x}_i^\top \mathbf{t}_j)/\tau)} - \log\frac{\exp(\mathbf{t}_i^\top \mathbf{x}_i/\tau)}{\sum_j \exp(\mathbf{t}_i^\top \mathbf{x}_j)/\tau)}\]
where $\tau$ is the temperature. 

This contrastive loss aims to pull the image and its description together while pushing away the irrelevant text~\cite{softtriple}, which enables the emerging multi-modal applications, e.g., zero-shot transfer~\cite{radford2021learning,inmap}, text-to-image generation~\cite{sdiffusion}, etc. 

\subsection{Multi-modal Proxy Learning}\label{sec:method_proxy}
Given the pre-trained vision and text encoders from CLIP, this work takes one step further to investigate if we can extract user-specific information from the alignment between images and text.

Concretely, given an image of fruit~\cite{hu2017finding} as shown in Fig.~\ref{fig:intro}, some users may be interested in only one specific property of the object, e.g., color. In this scenario, applying the vision encoder to extract the representation for the whole image can miss the preference of users. To mitigate the problem, we propose to explore the proxy representation from the image with the guidance from the text using users' preference, named \underline{Multi}-\underline{M}od\underline{a}l \underline{P}roxy learning (\sysname{}).

Recall that CLIP is pre-trained by images and text descriptions, where the text prompt is ``a photo of a fruit'' for an image containing ``fruit''. Now given a user's preference (e.g., color), we can rewrite the prompt as ``fruit with the color of *'' denoted by $t_i^*$ for image $x_i$, where ``*'' is the proxy word and its token embedding is $\mathbf{w}_i$ that is learnable. Then, we can align image and text prompt representations to obtain the appropriate proxy embedding for a user's interest. Since there is no negative pairs, only the similarity between positive pairs can be optimized as
\begin{eqnarray}\label{eq:cos_loss}
\mathcal{L}(\mathbf{w}_i) = - \langle f(x_i), h(t_i^*)\rangle
\end{eqnarray}
where the vision and text encoders are frozen, and $\mathbf{w}_i$ is the only variable for learning as the representation of the proxy word. By maximizing the similarity to the image representation, \sysname{} aims to learn the optimal text proxy according to the user's interest. 

However, it should be noted that the text encoder was pre-trained with discrete text tokens, while the domain of $\mathbf{w}_i$ in Eqn.~\ref{eq:cos_loss} is unconstrained. Therefore, the text representation extracted from the frozen text encoder can be inaccurate for $\mathbf{w}_i$ that degenerates the performance, which is demonstrated in the following theory.

For the sake of simplicity, we assume $h'(t)\in\mathbb{R}$ is defined on the whole set but only has the right estimation on a discrete set as $T=\{t_i\}$ and the counterpart with the unconstrained set is denoted as $H(w)$. According to the definition, we have $\forall t\in T, h'(t) = H(t)$. The gap between the estimation from $h'$ and $H$ on unconstrained variable $w$ can be depicted as follows.

\begin{thm}\label{thm:1}
Given $w\not\in T$ and $t\in T$, if assuming $h'$ and $H$ are $L_h$ and $L_H$-Lipschitz continuous, we have
\[\|h'(w) - H(w)\|_2 \leq (L_h+L_H)\|t-w\|_2\]
\end{thm}
\begin{proof}
According to the definition, we have
\begin{align*}
&\|h'(w) - H(w)\|_2 = \|h'(w) - h'(t)+h'(t) - H(w)\|_2\\
&\leq \|h'(w) - h'(t)\|_2 + \|h'(t) - H(w)\|_2\\
&=\|h'(w) - h'(t)\|_2 + \|H(t) - H(w)\|_2\\
&\leq (L_h+L_H)\|t-w\|_2
\end{align*}
\end{proof}

\paragraph{Remark} Theorem~\ref{thm:1} implies that the distance of the estimation $h'(w)$ to the ground-truth result $H(w)$ is bounded by that of $w$ to an arbitrary discrete token $t$. Therefore, by selecting the nearest token as the reference, the bound can be improved as shown in the following corollary.

\begin{cor}
With the assumptions in Theorem~\ref{thm:1} and letting $t' = \arg\min_i \|t_i-w\|_2$, we have
\[\|h'(w) - H(w)\|_2 \leq (L_h+L_H)\|t'-w\|_2\]
\end{cor}

\subsubsection{Concept-level Constraint}\label{sec:method_concept}
According to the above analysis, a good reference $t'$ can help guarantee the performance. Fortunately, the input concept (e.g., color) from the user can be leveraged as the reference to constrain the freedom of the proxy word. Therefore, given the target concept word $u$, we can obtain its token embedding as $\mathbf{u}=\phi(u)$. Then, to learn appropriate representations from the proxy embedding, the original problem can be rewritten with the constraint as
\[\mathcal{L}(\mathbf{w}_i) = - \langle f(x_i), h(t_i^*)\rangle\quad s.t. \quad \|\mathbf{w}_i - \mathbf{u}\|_2^2\leq \lambda\]
The constrained problem is equivalent to
\begin{eqnarray}\label{eq:ref}
\mathcal{L}(\mathbf{w}_i) = - \langle f(x_i), h(t_i^*)\rangle + \alpha \|\mathbf{w}_i - \mathbf{u}\|_2^2
\end{eqnarray}
following~\cite{convex}, which can be optimized effectively by gradient descent.
\vspace{-0.2cm}
\subsubsection{Constrained Optimization with Reference Word}\label{sec:method_reference}

However, it is well known that the user concept is often with a large scope covering a broad range of words (e.g., color covers all including but not limited to `red', `blue', `green', etc.). As suggested by our above theoretical analysis, it is desired if the reference is as close as possible. In a clustering scenario and given only the user's high-level concept, it is challenging to find a closer reference word to further constrain the proxy learning. Fortunately, with the development of large language models (LLMs), we can leverage them (e.g., GPT-4) to provide relevant words according to a user's high-level concept as the candidate set and develop a selection strategy to obtain a closer reference word for each image. While the responses gathered from GPT-4 might not always precisely align with the data's ground truth, they indisputably furnish valuable candidate features, enriching the capabilities of \sysname{}.

To elucidate, considering the task of clustering a fruit dataset based on the concept of color, we can pose a query to GPT-4 as ``What are the common colors of fruit?''. The response we obtain is ``Common colors of fruit include red, yellow, green, orange, purple, and blue''. The colors enumerated in this response can serve as reference word candidates. It is worth noting that while the color spectrum derived from this method might surpass the actual colors present in the data, we collect all of them into the candidate set of reference words $\{z_k\}_k$, where $z_1$: ``red'', $z_2$: ``yellow'', etc. Then, their text representations can be obtained from the prompt $t_k$ as ``fruit with the color of $z_k$''. Given the image $x_i$, the closest reference can be observed as
\[z_i = \arg\max_k \langle \mathbf{x}_i, \mathbf{t}_k\rangle\]
where $\mathbf{t}_k=h(t_k)$. After that, $\mathbf{w}_i$ can be initialized with the token embedding of $z_i$ as $\mathbf{z}_i=\phi(z_i)$. Moreover, we change the regularization using a closer reference word compared to the high-level concept as
\begin{equation}\label{eq:loss_c_l2}
   \mathcal{L}(\mathbf{w}_i) = - \langle f(x_i), h(t_i^*)\rangle + \alpha\|\mathbf{w}_i - \mathbf{z}_i\|_2^2 
\end{equation}

\begin{table}[t]
   \centering
   \resizebox{0.67\columnwidth}{!}{
       \begin{tabular}{cc cc}
       \toprule
        {Datasets}    &   \# Samples & \# Clusters \\
        \midrule
        {ALOI} & 288  & 2;2 \\
        {Card} & 8,029  & 13;4 \\
        {CMUface} & 640  & 4;20;2;4 \\
        {Fruit}  & 105  & 3;3 \\
        {Fruit360} & 4,856  & 4;4 \\
        {Stanford Cars} & 1,200  & 4;3\\
        {Flowers} & 1,600  & 4;4  \\
        \bottomrule
       \end{tabular}
   }
   \vspace{-0.3cm}
   \caption{Dataset Statistics.}
   \vspace{-0.5cm}
   \label{tab:dataset}
\end{table}
\begin{table*}[t]
    \centering
    \resizebox{0.85\textwidth}{!}{
    \begin{tabular}{cc|ccccccccccccc}
    \toprule
         \multirow{2}{*}{Dataset} & \multirow{2}{*}{Clustering} & \multicolumn{2}{c}{MSC} & \multicolumn{2}{c}{MCV} & \multicolumn{2}{c}{ENRC} & \multicolumn{2}{c}{iMClusts} & \multicolumn{2}{c}{AugDMC}  & \multicolumn{2}{c}{\sysname{}} \\ 
        ~ & ~ & NMI & RI & NMI & RI & NMI & RI & NMI & RI & NMI & RI  & NMI & RI \\ 
        \midrule
        \multirow{2}{*}{ALOI} & Color & 0.1563 & 0.3428 & 0.6982 & 0.7439 & 0.9833 & 0.9892 & 1.0000 & 1.0000 & 1.0000 & 1.0000  & 1.0000 & 1.0000 \\ 
        ~ & Shape & 0.2968 & 0.5199 & 0.7359 & 0.8261 & 0.9732 & 0.9861 & 0.9963 & 0.9989 & 1.0000 & 1.0000  & 1.0000 & 1.0000 \\ 
        \midrule
        \multirow{2}{*}{Fruit} & Color & 0.6886 & 0.8051 & 0.6266 & 0.7685 & 0.7103 & 0.8511 & 0.7351 & 0.8632 & 0.8517 & 0.9108  & \textcolor{blue}{\textbf{0.8619}} & \textcolor{blue}{\textbf{0.9526}} \\ 
        ~ & Species & 0.1627 & 0.6045 & 0.2733 & 0.6597 & 0.3187 & 0.6536 & 0.3029 & 0.6743 & 0.3546 & 0.7399  & \textcolor{blue}{\textbf{1.0000}} & \textcolor{blue}{\textbf{1.0000}} \\ 
        \midrule
        \multirow{2}{*}{Fruit360} & Color & 0.2544 & 0.6054 & 0.3776 & 0.6791 & 0.4264 & 0.6868 & 0.4097 & 0.6841 & 0.4594 & 0.7392  & \textcolor{blue}{\textbf{0.6239}} & \textcolor{blue}{\textbf{0.8243}} \\ 
        ~ & Species & 0.2184 & 0.5805 & 0.2985 & 0.6176 & 0.4142 & 0.6984 & 0.3861 & 0.6732 & 0.5139 & 0.7430  & \textcolor{blue}{\textbf{0.5284}} & \textcolor{blue}{\textbf{0.7582}} \\ 
        \midrule
        \multirow{2}{*}{Card} & Order & 0.0807 & 0.7805 & 0.0792 & 0.7128 & 0.1225 & 0.7313 & 0.1144 & 0.7658 & 0.1440 & 0.8267  & \textcolor{blue}{\textbf{0.3653}} & \textcolor{blue}{\textbf{0.8587}} \\ 
        ~ & Suits & 0.0497 & 0.3587 & 0.0430 & 0.3638 & 0.0676 & 0.3801 & 0.0716 & 0.3715 & 0.0873 & 0.4228  & \textcolor{blue}{\textbf{0.2734}} & \textcolor{blue}{\textbf{0.7039}} \\ 
        \midrule
        \multirow{4}{*}{CMUface} & Emotion & 0.1284 & 0.6736 & 0.1433 & 0.5268 & 0.1592 & 0.6630 & 0.0422 & 0.5932 & 0.0161 & 0.5367  & \textcolor{blue}{\textbf{0.1786}} & \textcolor{blue}{\textbf{0.7105}} \\ 
        ~ & Glass & 0.1420 & 0.5745 & 0.1201 & 0.4905 & 0.1493 & 0.6209 & 0.1929 & 0.5627 & 0.1039 & 0.5361  & \textcolor{blue}{\textbf{0.3402}} & \textcolor{blue}{\textbf{0.7068}} \\ 
        ~ & Identity & 0.3892 & 0.7326 & 0.4637 & 0.6247 & 0.5607 & 0.7635 & 0.5109 & 0.8260 & 0.5875 & 0.8334  & \textcolor{blue}{\textbf{0.6625}} & \textcolor{blue}{\textbf{0.9496}} \\ 
        ~ & Pose & 0.3687 & 0.6322 & 0.3254 & 0.6028 & 0.2290 & 0.5029 & 0.4437 & 0.6114 & 0.1320 & 0.5517  & \textcolor{blue}{\textbf{0.4693}} & \textcolor{blue}{\textbf{0.6624}} \\ 
        \midrule
        \multirow{2}{*}{Stanford Cars} & Color & 0.2331 & 0.6158 & 0.2103 & 0.5802 & 0.2465 & 0.6779 & 0.2336 & 0.6552 & 0.2736 & 0.7525 & \textcolor{blue}{\textbf{0.7360}} & \textcolor{blue}{\textbf{0.9193}} \\ 
        ~ & Type & 0.1325 & 0.5336 & 0.1650 & 0.5634 & 0.2063 & 0.6217 & 0.1963 & 0.5643 & 0.2364 & 0.7356  & \textcolor{blue}{\textbf{0.6355}} & \textcolor{blue}{\textbf{0.8399}} \\ 
        \midrule
        \multirow{2}{*}{Flowers} & Color & 0.2561 & 0.5965 & 0.2938 & 0.5860 & 0.3329 & 0.6214 & 0.3169 & 0.6127 & 0.3556 & 0.6931  & \textcolor{blue}{\textbf{0.6426}} & \textcolor{blue}{\textbf{0.7984}} \\ 
        ~ & Species & 0.1326 & 0.5273 & 0.1561 & 0.6065 & 0.1894 & 0.6195 & 0.1887 & 0.6077 & 0.1996 & 0.6227  & \textcolor{blue}{\textbf{0.6013}} & \textcolor{blue}{\textbf{0.8103}} \\ 
        \bottomrule
    \end{tabular}
    
    }
    \vspace{-0.3cm}
    \caption{Quantitative comparison. The significantly best results with 95\% confidence are in bold.}
    \label{tab:quantitive}
\vspace{-0.3cm}
\end{table*}
\vspace{-0.3cm}

\subsubsection{Contrastive Concepts}
In some application scenarios, one user may need more than one clustering and provide high-level concepts as $\{u_j\}$, e.g., $u_1$:``color'', $u_2$:``species'', etc. For the concept ``color'', the irrelevant concept ``species'' can be leveraged as the negative constraint for the learning of proxy word. Concretely, let $u_w$ denote the target concept word, a contrastive loss can be adopted as regularization
\[R(\mathbf{w}_i) = -\log \frac{\exp(\mathbf{w}_i^\top \mathbf{u}_w)}{\sum_j \exp(\mathbf{w}_i^\top \mathbf{u}_j)}\]
and the final objective becomes
\begin{align}\label{eq:loss_c_c}
&\mathcal{L}(\mathbf{w}_i) = \nonumber\\
&- \langle f(x_i), h(t_i^*)\rangle + \alpha \|\mathbf{w}_i - \mathbf{z}_i\|_2^2 - \beta \log \frac{\exp(\mathbf{w}_i^\top \mathbf{u}_w)}{\sum_j \exp(\mathbf{w}_i^\top \mathbf{u}_j)} 
\end{align}
where the first term is to infer the user-specific feature, while the latter two terms constrain the proxy word to the reference words for the appropriate representation extraction from the pre-trained text encoder. The overall framework of Multi-MaP is illustrated in Fig.~\ref{fig:intro}.

\vspace{-0.2cm}
\section{Experiments}

To demonstrate our proposed method, we evaluate \sysname ~on all publicly available image datasets in multiple clustering, including 
{ALOI}~\cite{geusebroek2005amsterdam}, 
{Stanford Cars}~\cite{krause20133d},
{Card}~\cite{DBLP:conf/inns-dlia/YaoLRH23}, {CMUface}~\cite{gunnemann2014smvc},
{Flowers}~\cite{nilsback2008automated},
{Fruit}~\cite{hu2017finding}, and {Fruit360}~\cite{DBLP:conf/inns-dlia/YaoLRH23}
as summarized in Table~\ref{tab:dataset}. 

We compare \sysname{} against five state-of-the-art methods: 
\textbf{MSC}~\cite{hu2017finding} is a traditional multiple clustering method that uses hand-crafted features; \textbf{MCV}~\cite{guerin2018improving} leverages multiple feature extractors to represent different ``views’’ of the same data and employs a multi-input neural network to enhance clustering outcomes; \textbf{ENRC}~\cite{miklautz2020deep} is a deep multiple clustering method that integrates auto-encoder and clustering objective to generate different clusterings; \textbf{iMClusts}~\cite{ren2022diversified} makes use of the expressive representational power of deep autoencoders and multi-head attention to accomplish multiple clusterings; \textbf{AugDMC}~\cite{DBLP:conf/inns-dlia/YaoLRH23} leverages augmentations to learn different image representations to achieve multiple clustering.

\subsection{Experiment Setup}

We employ Adam and set momentum as $0.9$ to train the model for $1000$ epochs. All hyper-parameters are searched according to the loss score of \sysname{}, where the learning rate is searched in $\{0.1,0.05,0.01,0.005,0.001,0.0005\}$, weight decay is in $\{0.0005, 0.0001, 0.00005, 0.00001, 0\}$, $\alpha, \beta$ are in $\{0.0, 0.1, 0.2, \dots, 1.0\}$, and $\lambda$ is fixed as $1$ for all the experiments. For the non pre-trained methods, we perform k-means~\cite{lloyd1982least} 10 times due to its randomness and evaluate the average clustering performance using two quantitative metrics, that is, Normalized Mutual Information (NMI)~\cite{white2004performance} and Rand index (RI)~\cite{rand1971objective}. These measures range in $[0, 1]$, and higher scores imply more accurate clustering results.
The experiments are conducted with GPU NVIDIA GeForce RTX 2080 Ti. 

It should also be noted that some data are difficult to obtain corresponding candidate labels from GPT-4 or the labels do not provide semantic features, such as names. For example, for the identity clustering for the CMUface dataset~\cite{gunnemann2014smvc}, different identities represent different people and the semantic meaning of names should not affect the clustering results. In this case, we randomly extract 10 words from WordNet~\cite{fellbaum2010wordnet} as reference words, in order to make the candidate labels more distinctive. For instance, we randomly choose ``abstain, candid, function, haphazard, knot, luxury, nonchalance, pension, resilience, taciturn'' for the above scenario. Furthermore, all publicly available multiple clustering datasets provide each ground-truth clustering a high-level concept, e.g., `shape', `pose', etc. Therefore, in the experiment, we directly use them as users' preferences for our evaluation purposes.

\subsection{Performance Comparison}

In our experiments, after we obtain the proxy word embedding of each image for a desired concept, we feed them to k-means to obtain the corresponding clustering. Since k-means is random, we repeat ten times and the average performance is reported in
Table~\ref{tab:quantitive}. The best results are marked by bold numbers. 

We can observe that \sysname{} outperforms the baselines in all the cases, indicating the superiority of the proposed method. This also shows a strong generalization ability of the pre-trained model by CLIP, which can capture the features of data in different aspects.

\begin{table}[t!]
    \centering
    \resizebox{\columnwidth}{!}{
    \begin{tabular}{cc|cccccc}
    \toprule
        \multirow{2}{*}{Dataset} & \multirow{2}{*}{Clustering} & \multicolumn{2}{c}{CLIP$_\text{GPT}$} & \multicolumn{2}{c}{CLIP$_\text{label}$} & \multicolumn{2}{c}{\sysname{}} \\
  
        ~ & ~ & NMI & RI & NMI & RI & NMI & RI \\ 
        \midrule

        \multirow{2}{*}{ALOI} & Color  & 0.8581 & 0.9407 & 1.0000 & 1.0000 & 1.0000 & 1.0000 \\ 
        ~ & Shape  & 1.0000 & 1.0000 & 1.0000 & 1.0000 & 1.0000 & 1.0000 \\ 
        \midrule
        \multirow{2}{*}{Fruit} & Color  & 0.7912 & 0.9075 & \textcolor{blue}{\textbf{0.8629}} & \textcolor{blue}{\textbf{0.9780}} & 0.8619 & 0.9526 \\ 
        ~ & Species  & 0.9793 & 0.9919 & 1.0000 & 1.0000 & 1.0000 & 1.0000 \\ 
        \midrule
        \multirow{2}{*}{Fruit360} & Color  & 0.5613 & 0.7305 & 0.5746 & 0.7673 & \textcolor{blue}{\textbf{0.6239}} & \textcolor{blue}{\textbf{0.8243}} \\ 
        ~ & Species  & 0.4370 & 0.7552 & \textcolor{blue}{\textbf{0.5364}} & \textcolor{blue}{\textbf{0.7631}} & 0.5284 & 0.7582 \\ 
        \midrule
        \multirow{2}{*}{Card} & Order & 0.3518 & 0.8458 & 0.3518 & 0.8458 & \textcolor{blue}{\textbf{0.3653}} & \textcolor{blue}{\textbf{0.8587}} \\ 
        ~ & Suits  & 0.2711 & 0.6123 & 0.2711 & 0.6123 & \textcolor{blue}{\textbf{0.2734}} & \textcolor{blue}{\textbf{0.7039}} \\ 
        \midrule
        \multirow{4}{*}{CMUface} & Emotion & 0.1576 & 0.6532 & 0.1590 & 0.6619 & \textcolor{blue}{\textbf{0.1786}} & \textcolor{blue}{\textbf{0.7105}} \\ 
        ~ & Glass  & 0.2905 & 0.6869 & \textcolor{blue}{\textbf{0.4686}} & \textcolor{blue}{\textbf{0.7505}} & 0.3402 & 0.7068 \\ 
        ~ & Identity  & 0.1998 & 0.6388 & 0.2677 & 0.7545 & \textcolor{blue}{\textbf{0.6625}} & \textcolor{blue}{\textbf{0.9496}} \\ 
        ~ & Pose  & 0.4088 & 0.6473 & 0.4691 & 0.6409 & \textcolor{blue}{\textbf{0.4693}} & \textcolor{blue}{\textbf{0.6624}} \\ 
        \midrule
        \multirow{2}{*}{Stanford Cars} & Color  & 0.6539 & 0.8237 & 0.6830 & 0.8642 & \textcolor{blue}{\textbf{0.7360}} & \textcolor{blue}{\textbf{0.9193}} \\ 
        ~ & Type  & 0.6207 & 0.7931 & \textcolor{blue}{\textbf{0.6429}} & \textcolor{blue}{\textbf{0.8456}} & 0.6355 & 0.8399 \\ 
        \midrule
        \multirow{2}{*}{Flowers} & Color  & 0.5653 & 0.7629 & 0.5828 & 0.7836 & \textcolor{blue}{\textbf{0.6426}} & \textcolor{blue}{\textbf{0.7984}} \\ 
        ~ & Species  & 0.5620 & 0.7553 & \textcolor{blue}{\textbf{0.6019}} & 0.7996 & 0.6013 & \textcolor{blue}{\textbf{0.8103}} \\ 
        \bottomrule
    \end{tabular}
    }
    \vspace{-0.3cm}
    \caption{Variants of CLIP. The significantly best results with 95\% confidence are in bold.}
    \label{tab:clip}
\vspace{-0.5cm}
\end{table}
\begin{table*}[h]
   \centering
   \resizebox{0.85\textwidth}{!}{
        \begin{tabular}{c c|cc|cc|cc|cc|cc}
        \toprule
         & & \multicolumn{2}{c}{\sysname{}$_{p}$} & \multicolumn{2}{c}{\sysname{}$_{c}$} & \multicolumn{2}{c}{\sysname{}$_{r}$} & \multicolumn{2}{c}{\sysname{}$_{cr}$} & \multicolumn{2}{c}{\sysname{}} \\

        \midrule
        
        \multirow{4}{*}{Modules}& Proxy Learning & \multicolumn{2}{c}{\checkmark}& \multicolumn{2}{c}{\checkmark}& \multicolumn{2}{c}{\checkmark}& \multicolumn{2}{c}{\checkmark}  & \multicolumn{2}{c}{\checkmark}\\
        
         ~& Concept Word & \multicolumn{2}{c}{\scalebox{0.85}[1]{$\times$}}& \multicolumn{2}{c}{\checkmark}& \multicolumn{2}{c}{\scalebox{0.85}[1]{$\times$}}& \multicolumn{2}{c}{\checkmark} & \multicolumn{2}{c}{\checkmark}\\

        ~ & Reference Word & \multicolumn{2}{c}{\scalebox{0.85}[1]{$\times$}}& \multicolumn{2}{c}{\scalebox{0.85}[1]{$\times$}}& \multicolumn{2}{c}{\checkmark}& \multicolumn{2}{c}{\checkmark} & \multicolumn{2}{c}{\checkmark}\\

        ~ & Contrastive Concepts & \multicolumn{2}{c}{\scalebox{0.85}[1]{$\times$}}& \multicolumn{2}{c}{\scalebox{0.85}[1]{$\times$}}& \multicolumn{2}{c}{\scalebox{0.85}[1]{$\times$}}& \multicolumn{2}{c}{\scalebox{0.85}[1]{$\times$}} & \multicolumn{2}{c}{\checkmark}\\

        \midrule

        ~ & ~ & NMI$\uparrow$& RI$\uparrow$& NMI$\uparrow$& RI$\uparrow$& NMI$\uparrow$& RI$\uparrow$ & NMI$\uparrow$& RI$\uparrow$&NMI$\uparrow$& RI$\uparrow$\\

        \midrule

        \multirow{2}{*}{ALOI~\cite{geusebroek2005amsterdam}} & Color & 0.9619 & 0.9826 & 1.0000 & 1.0000 & 0.9795 & 0.9869 & 1.0000 & 1.0000 & 1.0000 & 1.0000 \\ 
        ~ & Shape & 1.0000 & 1.0000 & 1.0000 & 1.0000 & 1.0000 & 1.0000 & 1.0000 & 1.0000 & 1.0000 & 1.0000 \\ 
        \midrule
        \multirow{2}{*}{Fruit~\cite{hu2017finding}} & Color & 0.7642 & 0.8439 & 0.8215 & 0.9283 & 0.8136 & 0.9073 & 0.8484 & 0.9308 & \textcolor{blue}{\textbf{0.8619}} & \textcolor{blue}{\textbf{0.9526}} \\ 
        ~ & Species & 1.0000 & 1.0000 & 1.0000 & 1.0000 & 1.0000 & 1.0000 & 1.0000 & 1.0000 & 1.0000 & 1.0000 \\ 
        \midrule
        \multirow{2}{*}{Fruit360~\cite{DBLP:conf/inns-dlia/YaoLRH23}} & Color & 0.5643 & 0.7665 & 0.6217 & 0.7836 & 0.5910 & 0.7746 & 0.6089 & 0.7965 & \textcolor{blue}{\textbf{0.6239}} & \textcolor{blue}{\textbf{0.8243}} \\ 
        ~ & Species & 0.5077 & 0.7368 & 0.5137 & 0.7436 & 0.5094 & 0.7425 & 0.5199 & 0.7428 & \textcolor{blue}{\textbf{0.5284}} & \textcolor{blue}{\textbf{0.7582}} \\ 
        \midrule
        \multirow{2}{*}{Card~\cite{DBLP:conf/inns-dlia/YaoLRH23}} & Order & 0.1932 & 0.8152 & 0.3568 & 0.8472 & 0.3113 & 0.8229 & 0.3616 & 0.8094 & \textcolor{blue}{\textbf{0.3653}} & \textcolor{blue}{\textbf{0.8587}} \\ 
        ~ & Suits & 0.2375 & 0.6282 & 0.2696 & 0.6641 & 0.2498 & 0.6365 & 0.2562 & 0.6599 & \textcolor{blue}{\textbf{0.2734}} & \textcolor{blue}{\textbf{0.7039}} \\ 
        \midrule
        \multirow{4}{*}{CMUface~\cite{gunnemann2014smvc}} & Emotion & 0.1690 & 0.6170 & 0.1714 & 0.6229 & 0.1697 & 0.6360 & 0.1713 & 0.6843 & \textcolor{blue}{\textbf{0.1786}} & \textcolor{blue}{\textbf{0.7105}} \\ 
        ~ & Glass & 0.3112 & 0.6911 & 0.3269 & 0.7136 & 0.3162 & 0.6917 & 0.3370 & 0.7108 & \textcolor{blue}{\textbf{0.3402}} & \textcolor{blue}{\textbf{0.7068}} \\ 
        ~ & Identity & 0.5617 & 0.8234 & 0.6243 & 0.8359 & 0.5839 & 0.8263 & 0.6391 & 0.8946 & \textcolor{blue}{\textbf{0.6625}} & \textcolor{blue}{\textbf{0.9496}} \\ 
        ~ & Pose & 0.4361 & 0.6386 & 0.4550 & 0.6499 & 0.4381 & 0.6429 & 0.4387 & 0.6489 & \textcolor{blue}{\textbf{0.4693}} & \textcolor{blue}{\textbf{0.6624}} \\ 
        \midrule
        \multirow{2}{*}{Stanford cars~\cite{krause20133d}} & Color & 0.5939 & 0.7835 & 0.6836 & 0.8659 & 0.6729 & 0.8638 & 0.7112 & 0.9117 & \textcolor{blue}{\textbf{0.7360}} & \textcolor{blue}{\textbf{0.9193}} \\ 
        ~ & Type & 0.5569 & 0.7996 & 0.6383 & 0.8271 & 0.6091 & 0.8046 & 0.6289 & 0.8181 & \textcolor{blue}{\textbf{0.6355}} & \textcolor{blue}{\textbf{0.8399}} \\ 
        \midrule
        \multirow{2}{*}{Flowers~\cite{nilsback2008automated}} & Color & 0.5783 & 0.7723 & 0.5830 & 0.7833 & 0.5987 & 0.7849 & 0.6216 & 0.7941 & \textcolor{blue}{\textbf{0.6426}} & \textcolor{blue}{\textbf{0.7984}} \\ 
        ~ & Species & 0.5704 & 0.7608 & 0.5744 & 0.7842 & 0.5723 & 0.7811 & 0.5846 & 0.7892 & \textcolor{blue}{\textbf{0.6013}} & \textcolor{blue}{\textbf{0.8103}} \\ 
	\bottomrule
	\end{tabular}
   }
   \caption{Components ablation. All of our components boost performance consistently in all benchmark multi-clustering vision tasks.}
	\label{tab:ablation_supp}
\end{table*}

Since our method uses the CLIP encoder and GPT-4 to obtain clustering results,
a natural question arises how the performance would be if we directly use them in a zero-short manner.
Therefore, we provide two zero-shot variants of CLIP, that is, 
{CLIP$_\text{GPT}$} that uses GPT-4 to obtain the candidate labels and predicts labels through zero-shot classification with all candidate labels as class names, and
{CLIP$_\text{label}$} that performs zero-shot classification with all ground truth labels. It should be noted that CLIP$_\text{label}$ uses an unfair setting with a ground-truth label set known in advance, which is expected to provide the best performance using CLIP in a zero-short manner. The results are shown in Table~\ref{tab:clip}.

As expected, CLIP$_\text{label}$ achieves better performance than CLIP$_\text{GPT}$ in almost all cases, since CLIP$_\text{label}$ uses a fixed ground truth label set, while CLIP$_\text{GPT}$ uses candidate labels that may not match the ground truth exactly, introducing noise. Note that CLIP$_\text{GPT}$ and CLIP$_\text{label}$ achieve the same results on Cards, because the candidate labels provided by GPT-4 are exactly the same as the true labels.

Besides, \sysname{} performs better than CLIP$_\text{GPT}$ in almost all cases, indicating that the proposed method can learn more effective features through its training process. Moreover, although CLIP$_\text{label}$ uses the ground truth, which is expected to be the best, \sysname{} still outperforms CLIP$_\text{label}$ in some cases, such as the clustering of color for Fruit360 dataset. This is because CLIP is more inclined to capture the features of one aspect of the data, while \sysname{} learns better embedding of different aspects by training with the supervision of users' interests. Furthermore, \sysname{} also achieves very competitive results as CLIP$_\text{label}$ in the remaining cases, which further demonstrates the effectiveness of the proposed method.

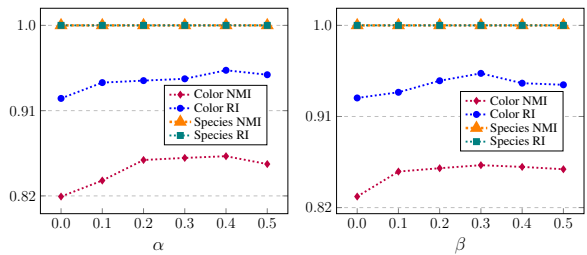
\begin{figure}[t!]
\centering
    \subfloat[Parameter sensitivity of $\alpha$]{
            \begin{tikzpicture}[font=\Large, scale=0.48]
                \begin{axis}[
                    legend cell align={left},
                    legend style={nodes={scale=0.7, transform shape}, at={(0.5,0.48)},anchor=west},
                    xlabel={$\alpha$\color{white}{$\beta$}},
                    xtick pos=left,
                    tick label style={font=\large},
                    ylabel style={font=\large},
                    xtick={1, 2, 3, 4, 5, 6},
                    xticklabels={$0.0$,$0.1$,$0.2$,$0.3$,$0.4$,$0.5$},
                    ytick={0.82, 0.91,1.0},
                    yticklabels={ $0.82$,$0.91$,$1.0$},
                    ymajorgrids=true,
                    grid style=dashed
                ]
                \addplot[
                    color=purple,
                    dotted,
                    mark options={solid},
                    mark=diamond*,
                    line width=1.5pt,
                    mark size=2pt
                    ]
                    coordinates {
                    (1, 0.8193)
                    (2, 0.8362)
                    (3, 0.8579)
                    (4, 0.8601)
                    (5, 0.8619)
                    (6, 0.8536)
                    };
                    \addlegendentry{Color NMI}
                \addplot[
                    color=blue,
                    dotted,
                    mark options={solid},
                    mark=*,
                    line width=1.5pt,
                    mark size=2pt
                    ]
                    coordinates {
                    (1, 0.9229)
                    (2, 0.9396)
                    (3, 0.9417)
                    (4, 0.9436)
                    (5, 0.9526)
                    (6, 0.9479)
                    };
                    \addlegendentry{Color RI}
                \addplot[
                    color=orange,
                    dotted,
                    mark options={solid},
                    mark=triangle*,
                    line width=2pt,
                    mark size=4pt
                    ]
                    coordinates {
                    (1, 1)
                    (2, 1)
                    (3, 1)
                    (4, 1)
                    (5, 1)
                    (6, 1)
                    };
                    \addlegendentry{Species NMI}
                \addplot[
                    color=teal,
                    dotted,
                    mark options={solid},
                    mark=square*,
                    line width=1.5pt,
                    mark size=2pt
                    ]
                    coordinates {
                    (1, 1)
                    (2, 1)
                    (3, 1)
                    (4, 1)
                    (5, 1)
                    (6, 1)
                    };
                    \addlegendentry{Species RI}
                \end{axis}
                \end{tikzpicture}
    }\subfloat[Parameter sensitivity of $\beta$]{
            \begin{tikzpicture}[font=\Large, scale=0.48]
                \begin{axis}[
                    legend cell align={left},
                    legend style={nodes={scale=0.7, transform shape}, at={(0.5,0.45)},anchor=west},
                    xlabel={$\beta$},
                    xtick pos=left,
                    tick label style={font=\large},
                    ylabel style={font=\large},
                    xtick={1, 2, 3, 4, 5, 6},
                    xticklabels={$0.0$,$0.1$,$0.2$,$0.3$,$0.4$,$0.5$},
                    ytick={0.82, 0.91,1.0},
                    yticklabels={ $0.82$,$0.91$,$1.0$},
                    ymajorgrids=true,
                    grid style=dashed
                ]
                \addplot[
                    color=purple,
                    dotted,
                    mark options={solid},
                    mark=diamond*,
                    line width=1.5pt,
                    mark size=2pt
                    ]
                    coordinates {
                    (1, 0.8309)
                    (2, 0.8556)
                    (3, 0.8589)
                    (4, 0.8619)
                    (5, 0.8602)
                    (6, 0.8579)
                    };
                    \addlegendentry{Color NMI}
                \addplot[
                    color=blue,
                    dotted,
                    mark options={solid},
                    mark=*,
                    line width=1.5pt,
                    mark size=2pt
                    ]
                    coordinates {
                    (1, 0.9283)
                    (2, 0.9339)
                    (3, 0.9453)
                    (4, 0.9526)
                    (5, 0.9429)
                    (6, 0.9413)
                    };
                    \addlegendentry{Color RI}
                \addplot[
                    color=orange,
                    dotted,
                    mark options={solid},
                    mark=triangle*,
                    line width=2pt,
                    mark size=4pt
                    ]
                    coordinates {
                    (1, 1)
                    (2, 1)
                    (3, 1)
                    (4, 1)
                    (5, 1)
                    (6, 1)
                    };
                    \addlegendentry{Species NMI}
                \addplot[
                    color=teal,
                    dotted,
                    mark options={solid},
                    mark=square*,
                    line width=1.5pt,
                    mark size=2pt
                    ]
                    coordinates {
                    (1, 1)
                    (2, 1)
                    (3, 1)
                    (4, 1)
                    (5, 1)
                    (6, 1)
                    };
                    \addlegendentry{Species RI}
                \end{axis}
                \end{tikzpicture}
    }
    \vspace{-0.3cm}
    \caption{Parameter analysis of $\alpha$ and $\beta$ on Fruit~\cite{hu2017finding}.}
    \label{fig:paramter1}
    \vspace{-0.6cm}
\end{figure}
\begin{figure*}[t]
\centering

    \subfloat[\small Color using ground truth]{
    \includegraphics[width=0.23\textwidth]{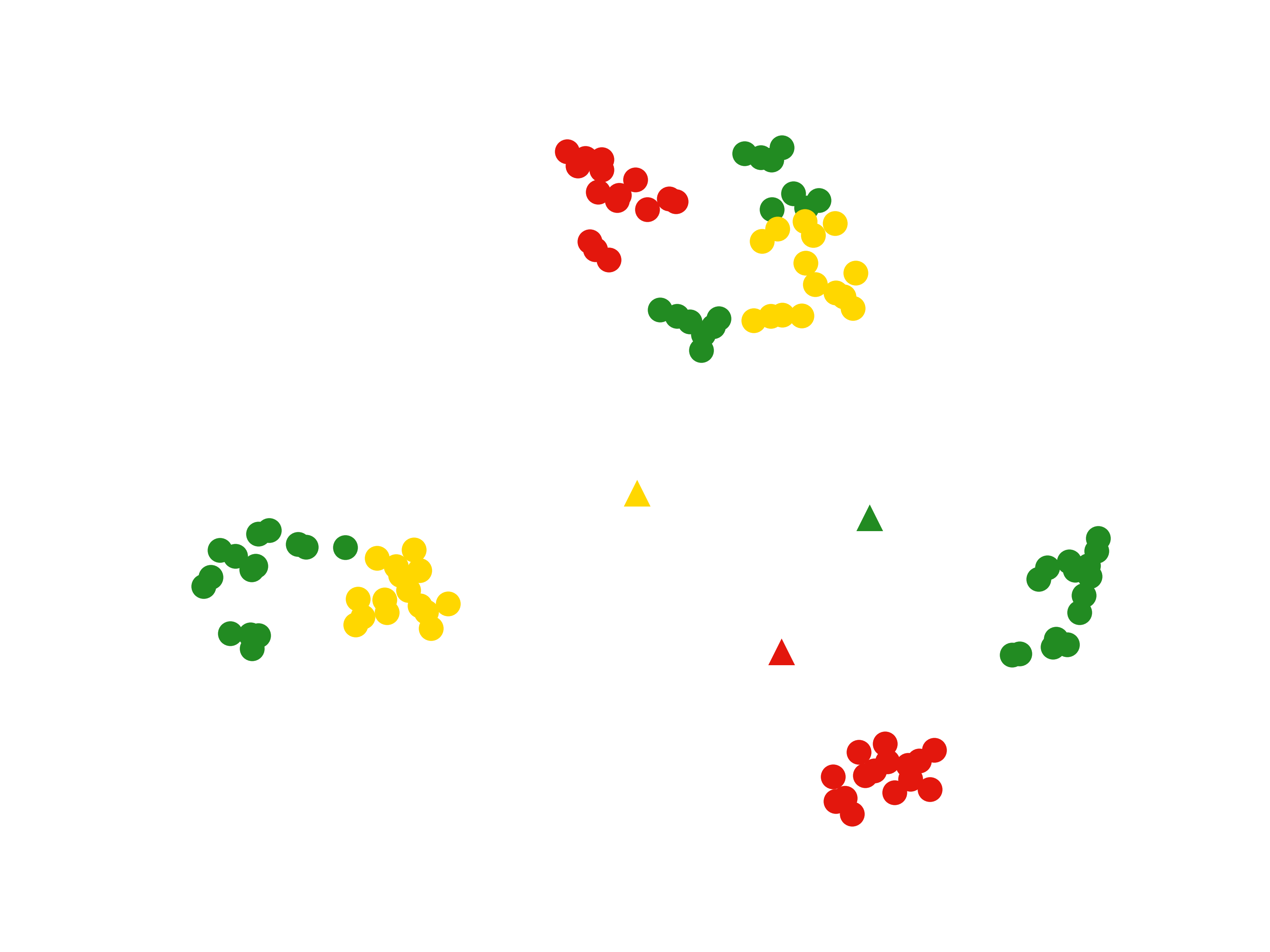}
    }
    \subfloat[\small Color using reference words]{
    \includegraphics[width=0.23\textwidth]{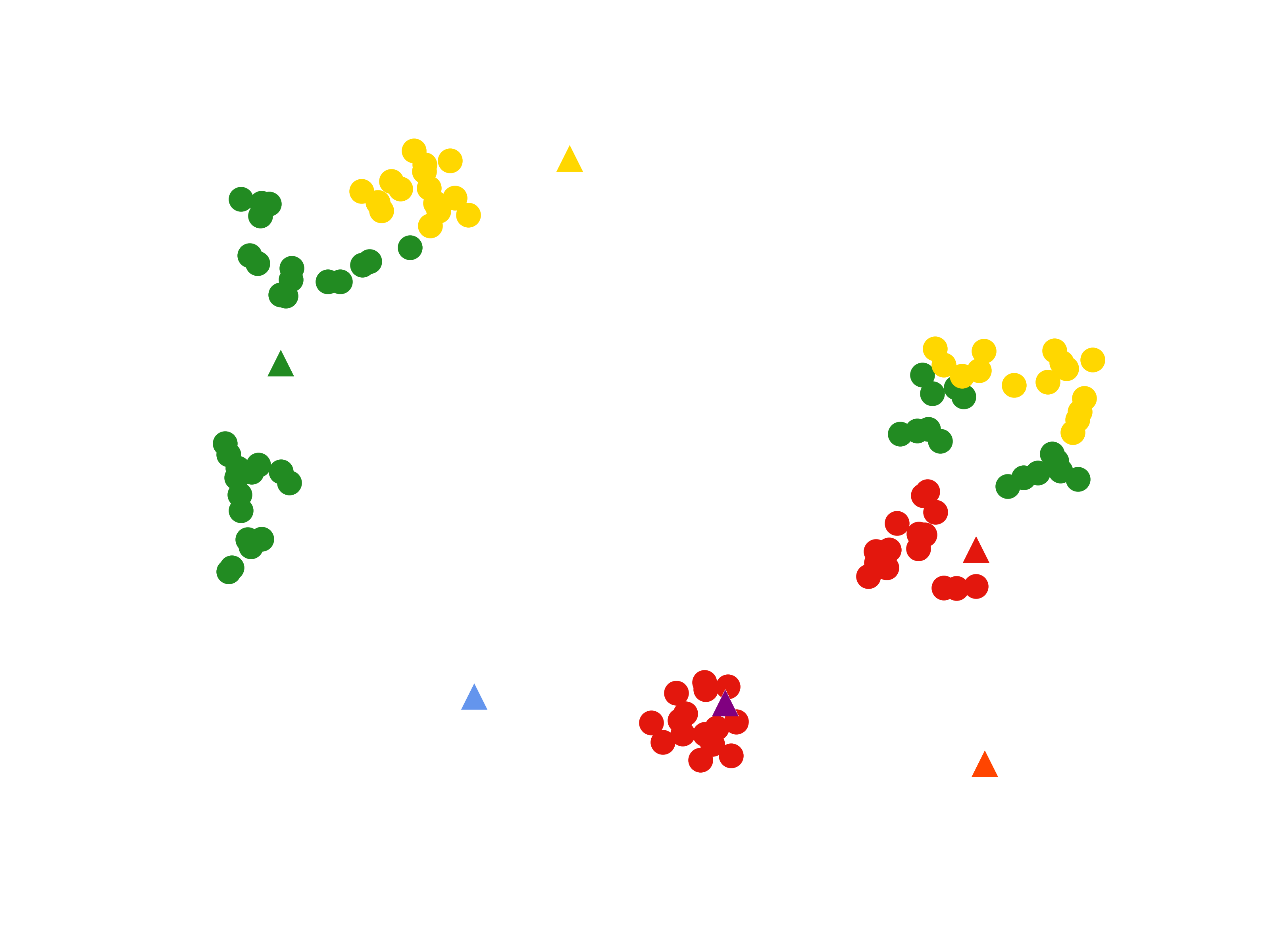}
    }
    \subfloat[\small Color of \sysname{}]{
    \includegraphics[width=0.23\textwidth]{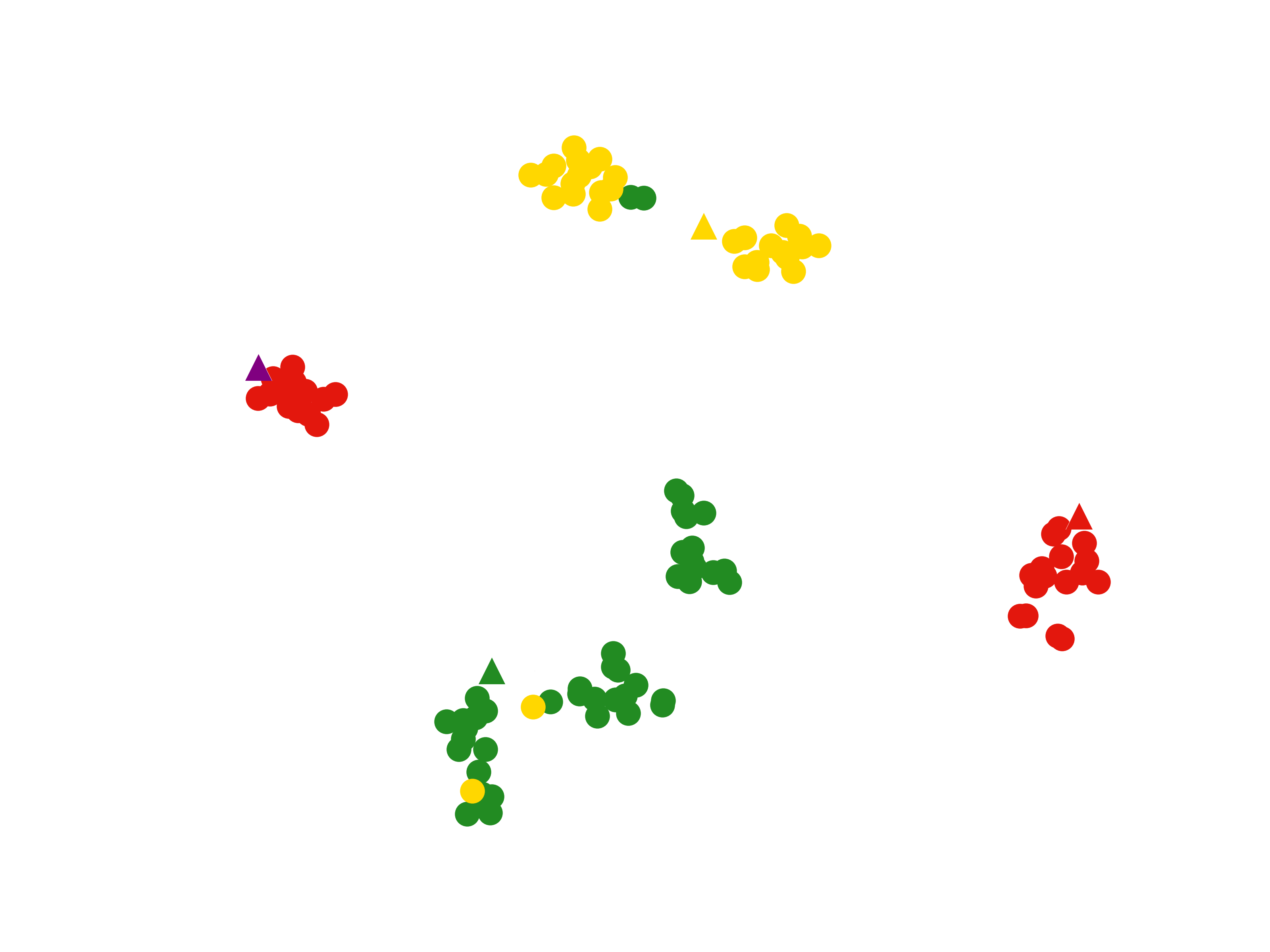}
    }
    \subfloat[\small Legend of color]{
    \includegraphics[width=0.2\textwidth]{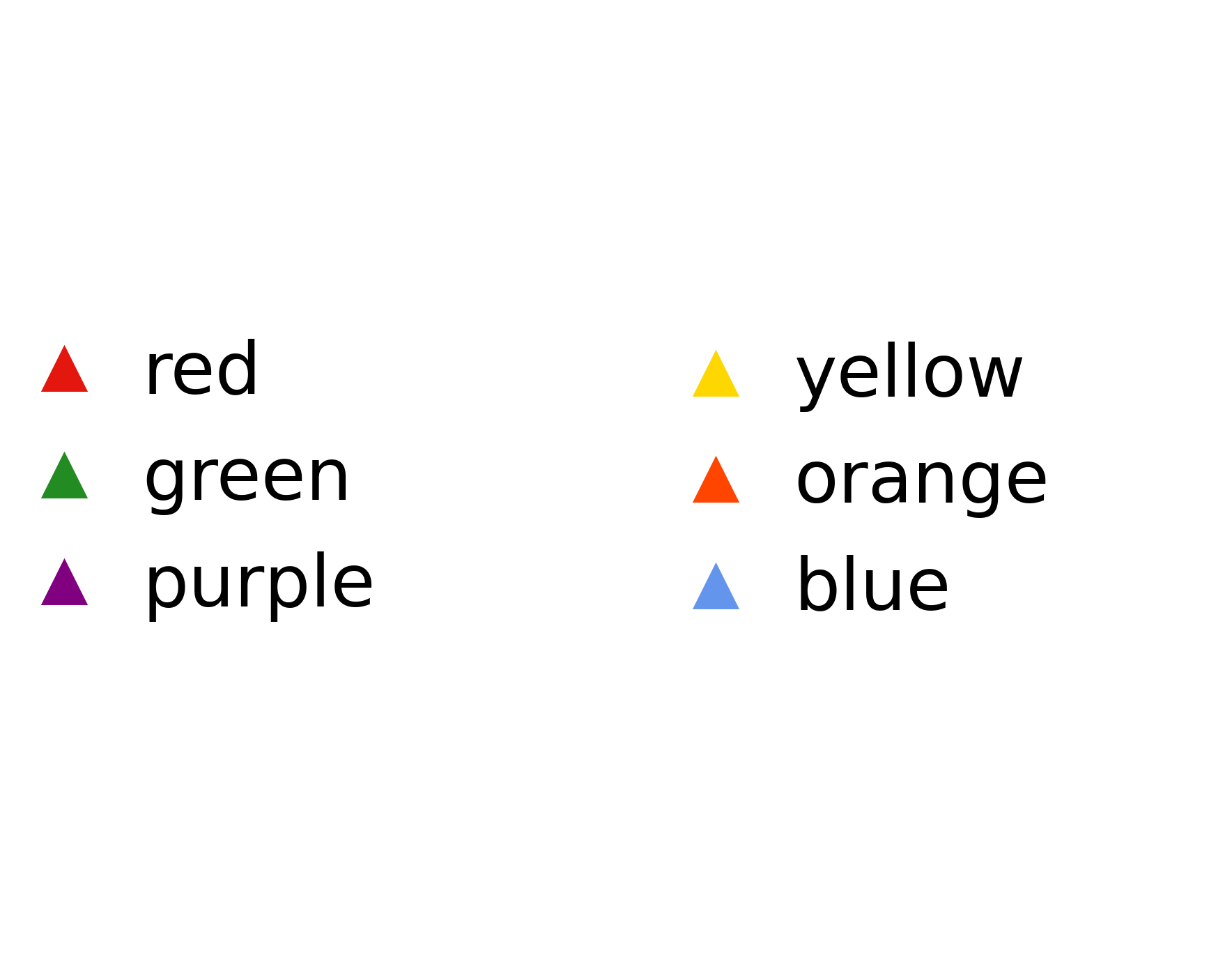}
    }
    \\
    
    \subfloat[\small  Species using ground truth]{
    \includegraphics[width=0.23\textwidth]{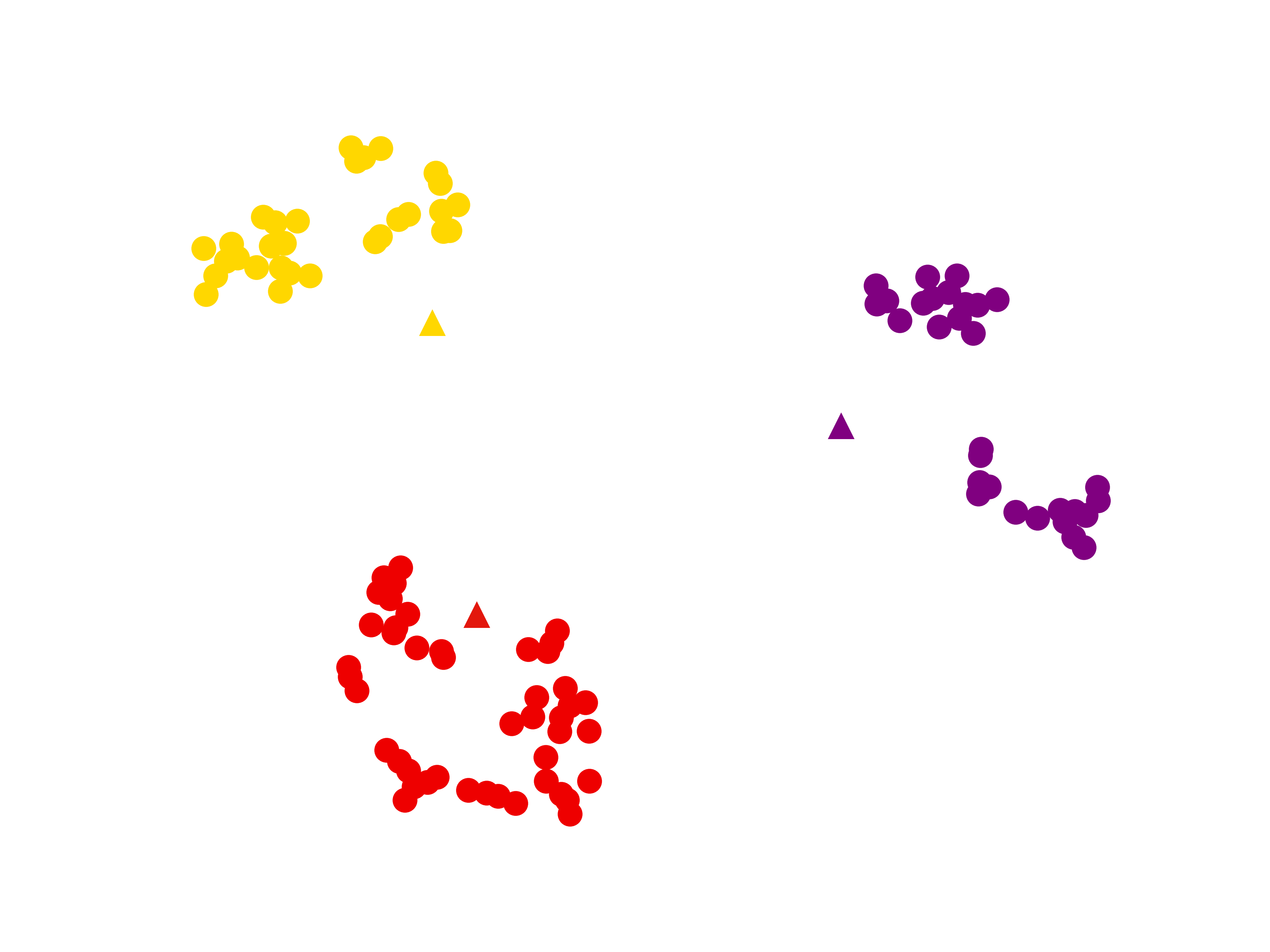}
    }
    \subfloat[\small Species using reference words]{
    \includegraphics[width=0.23\textwidth]{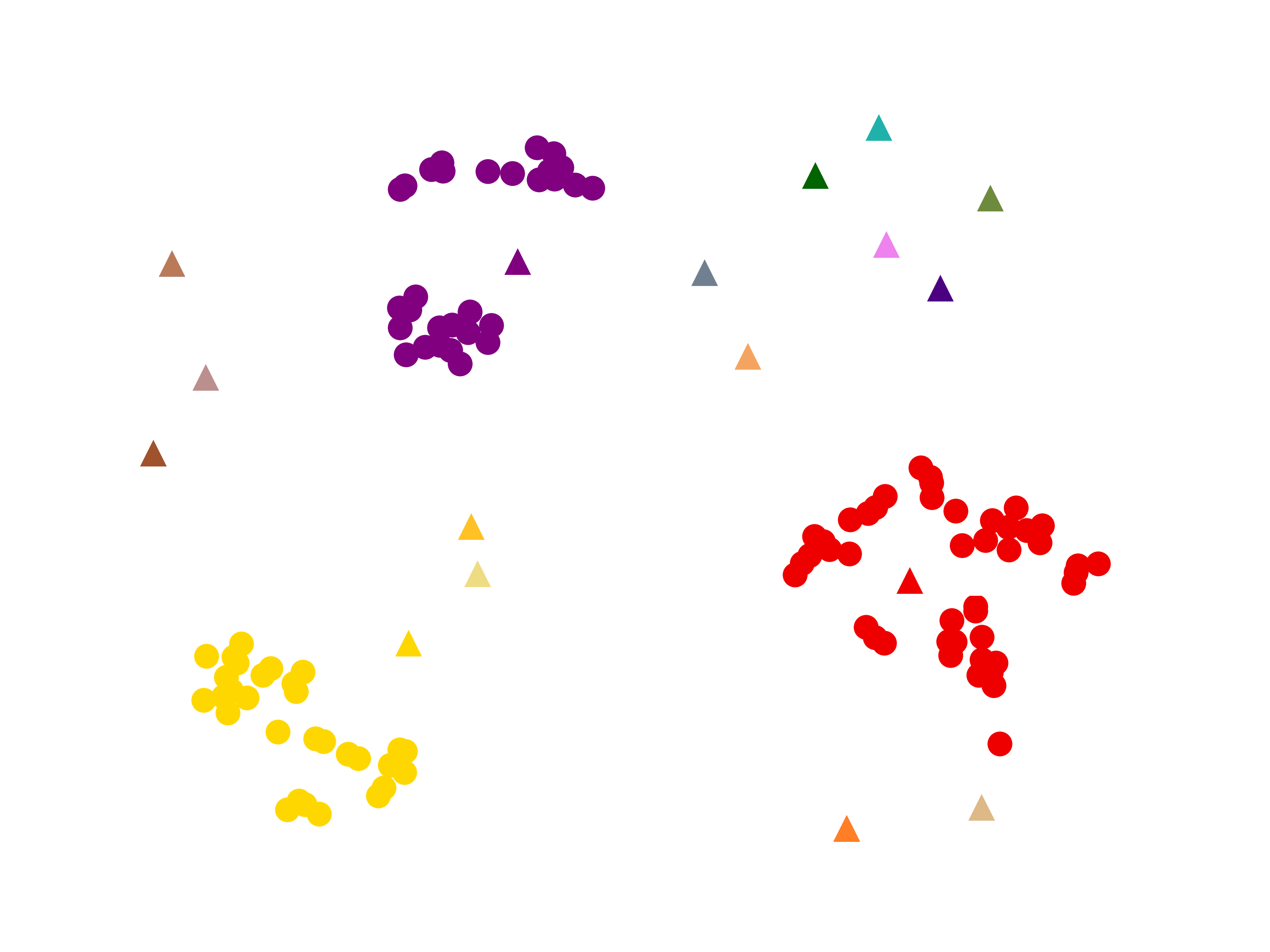}
    }
    \subfloat[\small Species of \sysname{}]{
    \includegraphics[width=0.23\textwidth]{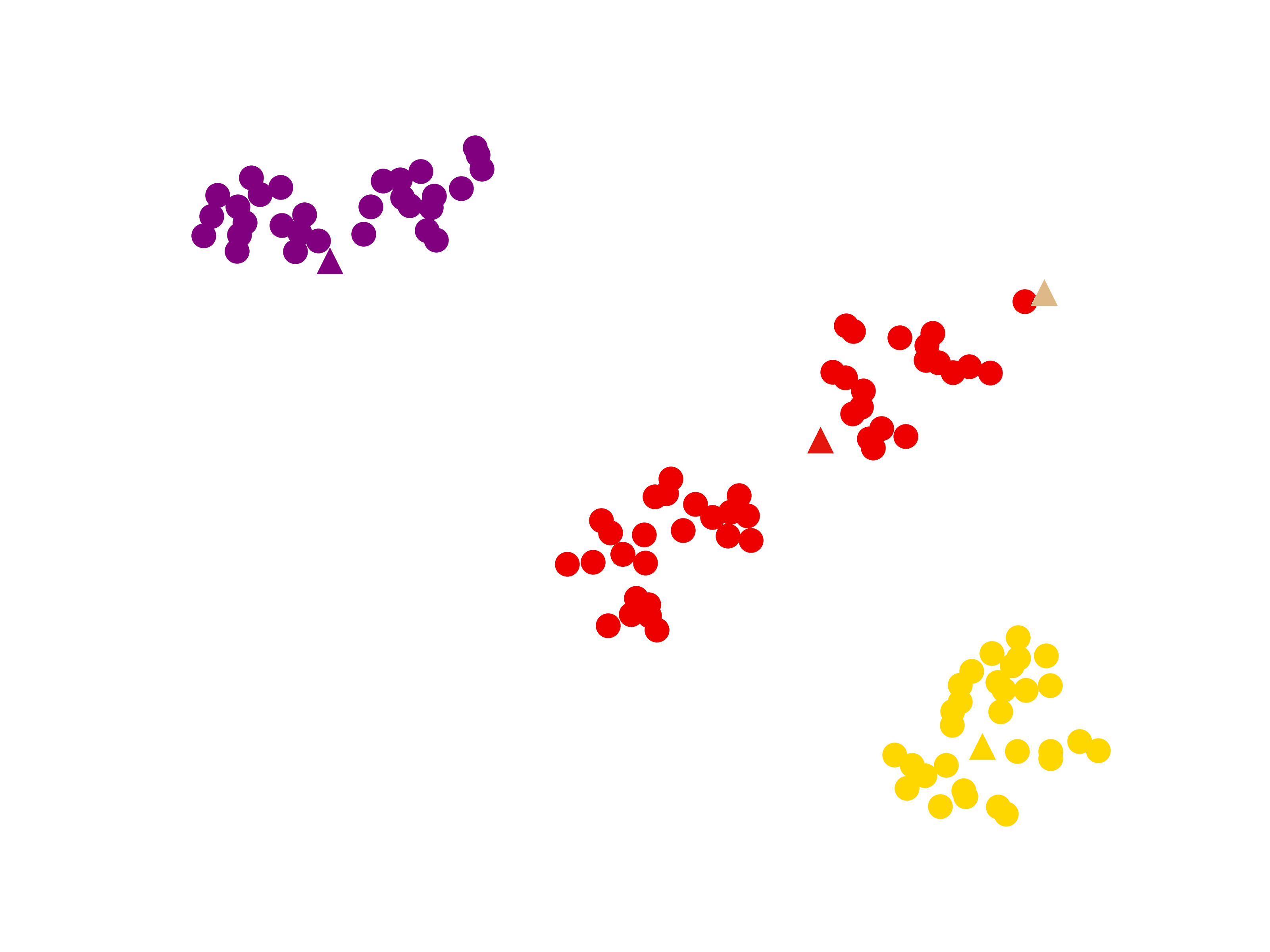}
    }
    \subfloat[\small Legend of species]{
    \includegraphics[width=0.2\textwidth]{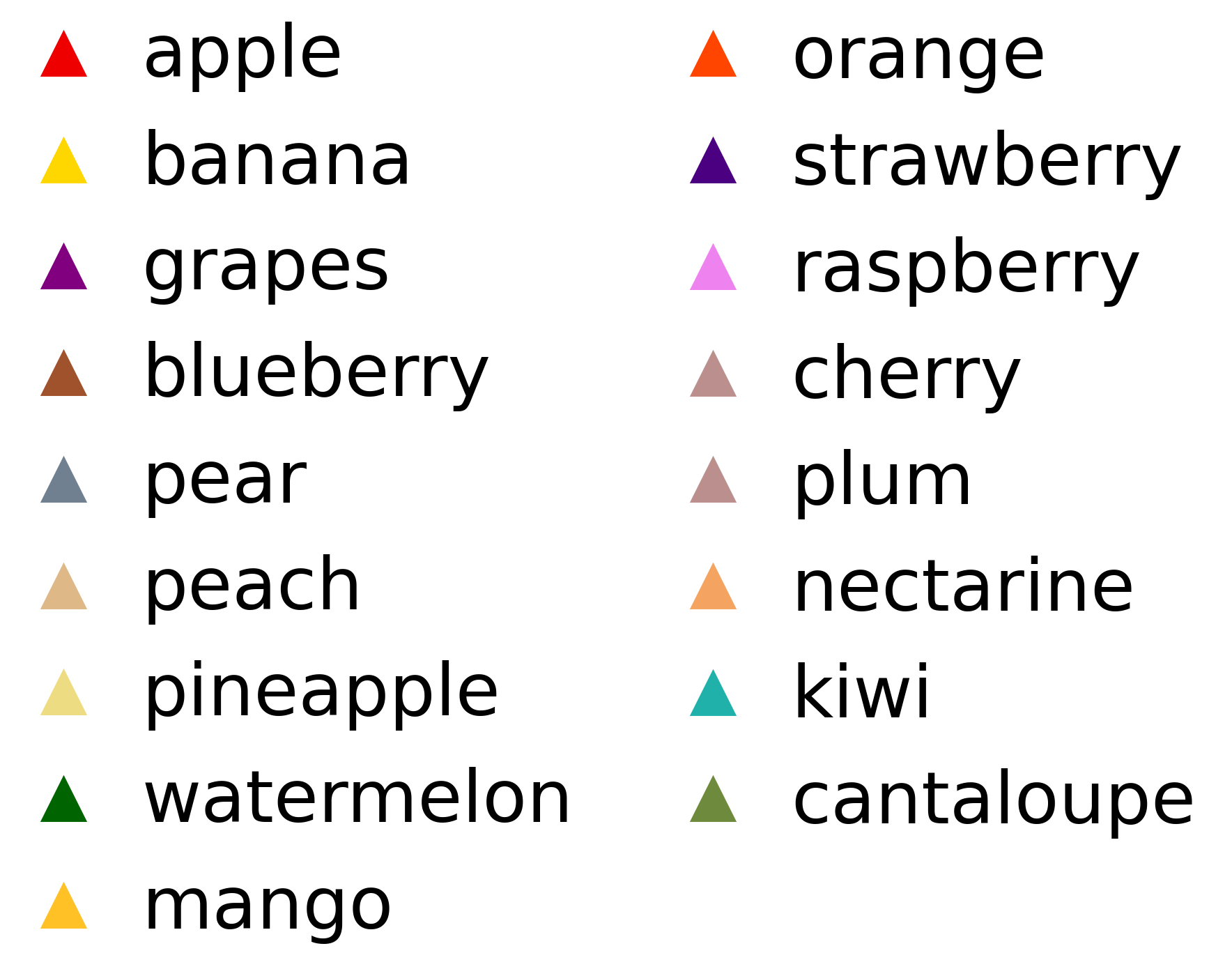}
    }

    \caption{Visualization of feature embeddings and related labels. The points represent the image or pseudo-word embeddings, and the triangles represent the prompt or label embeddings. Different colors represent different labels, which are indicated by the text next to the triangles.}
    \label{fig:visual}
    \vspace{-0.6cm}
\end{figure*}
\subsection{Ablation Study}

To validate the effectiveness of the proposed method, we show the gain from four components in \sysname{} in Table~\ref{tab:ablation_supp}. Let ``\sysname{}$_{p}$" denote the proxy learning without concept-level constraint, ``\sysname{}$_{c}$" denote the variant optimized by solely applying concept word to constrain the freedom of the proxy word, ``\sysname{}$_{r}$" denote the variant optimized by reference word provided by GPT-4 to find a closer reference word to further constrain the proxy learning, and ``\sysname{}$_{cr}$" denote the variant leveraged both concept word and reference word.

We can observe that the proxy learning only with concept word or reference word, i.e., \sysname{}$_{c}$ and \sysname{}$_{r}$, performs better than \sysname{}$_{p}$. This shows that the proposed concept-level constraint and constrained optimization with reference words play an important role in the model as demonstrated by our theoretical analysis. Moreover, the model with combined components, i.e, \sysname{}$_{cr}$, achieves better results than \sysname{}$_{r}$ and \sysname{}$_{c}$. This indicates the effectiveness of the combination of reference words and reference concepts. Finally, our proposal using all including the contrastive concepts can further improve the performance, and thus provide the best results on all cases. This further demonstrates our proposal.

\subsection{Parameter Analysis}

We further investigate the effect of the reference word constraint weight  $\alpha$ and concept-level constraint weight $\beta$ varying from $0.0$ to $0.5$. The results of \sysname{} on Fruit datasets are shown in Fig.~\ref{fig:paramter1} (a) and Fig.~\ref{fig:paramter1} (b), respectively. As $\alpha$ and $\beta$ change, the proposed method keeps a species score of $1$, since the image encoder can capture very effective species features. The above results also show that the proposed method can effectively capture useful information from images, reference words, and target concepts.
As $\alpha$ increases, the proposed method first increases and then decreases, and reaches the maximum values at $\alpha$=0.4. Similar results can be observed for $\beta$ that the performance of the proposed method first increases and then decreases as $\beta$ increases, and reaches the maximum values at $\beta$=0.3. Indicating a suitable value of $\alpha$ or $\beta$ is helpful for \sysname{} to obtain more effective embeddings for multiple clustering tasks. More studies such as efficiency analysis can be found in the supplementary.

\subsection{Visualization}

To further demonstrate the effectiveness of the proposed method, we visualize the representations obtained in CLIP$_\text{label}$, CLIP$_\text{GPT}$, and \sysname{}. Specifically, for CLIP$_\text{label}$ and CLIP$_\text{GPT}$, we visualize the image representations as well as prompts generated with real labels and candidate labels, respectively. For the \sysname{}, we visualize the word embedding $\boldsymbol{w}_{*}$ and the candidate labels selected for initialization. The results are shown in Fig.~\ref{fig:visual}. 
For species clustering, we can see that image embeddings show very clear boundaries and correspond well to the prompt for CLIP$_\text{label}$, which indicates that CLIP can effectively capture the features of the species in the data. CLIP$_\text{GPT}$ uses the candidate labels to generate prompts, which introduces more noise, but benefits from the CLIP text encoder, the image embedding can keep a relatively far distance from most of the irrelevant prompts. However, since there are still a few images that are labeled as peaches (i.e., a noisy label), it performs slightly worse than CLIP$_\text{label}$, as shown in Fig.~\ref{fig:visual}(e). Besides, \sysname{} can capture the image and users' interests in the training process, therefore it compensates for the shortcomings of CLIP$_\text{GPT}$ and achieves better results.
On the other hand, for color clustering, the prompts are farther away from CLIP$_\text{label}$ and CLIP$_\text{GPT}$, that indicates the image embeddings mainly capture the features of species, which have no direct connection with color. CLIP$_\text{GPT}$ generates the prompt from the candidate label, which has more noise than the ground truth label, resulting in worse performance than CLIP$_\text{label}$. The proposed method can distinguish different colors more clearly, because it can learn from the user's interest and capture the color-related features. However, some red color embeddings are closer to purple, because some images in the datasets are actually purple, but labeled as red.
To sum up, the proposed method can learn more effective embeddings based on the users' interests for multiple clustering tasks.
\vspace{-0.2cm}
\section{Conclusion}

\vspace{-0.1cm}
To conclude, our study thoroughly investigates the significant challenges that current advanced deep learning techniques face in multiple clustering. A key issue is that users often do not need every clustering result produced by an algorithm, and selecting the most relevant one requires an in-depth understanding of each outcome. To overcome this challenge, the proposed method introduces a novel multi-modal proxy learning process, which effectively aligns a user's brief keyword describing the interest with the corresponding vision components. By integrating a multi-modal model and GPT-4 to precisely capture a user's interest using a keyword, the proposed approach uses both reference word constraint and concept-level constraint to discover personalized clustering result(s), which can also lead to enhanced performance. Experiments on diverse datasets demonstrate the superiority of the proposed method in multiple clustering tasks with a precise capture of a user's interest. The proposed method is limited by data with semantic meaningful labels, although we can use WordNet to help, whose comprehensive study will be our future work.
\vspace{-0.2cm}
\section{Acknowledgement}
This research is supported in part by Advata Gift funding. All opinions, findings, conclusions and recommendations in this paper are those of the author and do not necessarily reflect the views of
the funding agencies.

{
    \small
    \bibliographystyle{ieeenat_fullname}
    \bibliography{ie}

\begin{thebibliography}{36}
\providecommand{\natexlab}[1]{#1}
\providecommand{\url}[1]{\texttt{#1}}
\expandafter\ifx\csname urlstyle\endcsname\relax
  \providecommand{\doi}[1]{doi: #1}\else
  \providecommand{\doi}{doi: \begingroup \urlstyle{rm}\Url}\fi

\bibitem[Bae and Bailey(2006)]{bae2006coala}
Eric Bae and James Bailey.
\newblock Coala: A novel approach for the extraction of an alternate clustering of high quality and high dissimilarity.
\newblock In \emph{ICDM}, pages 53--62. IEEE, 2006.

\bibitem[Bishop and Nasrabadi(2006)]{bishop2006pattern}
Christopher~M Bishop and Nasser~M Nasrabadi.
\newblock \emph{Pattern recognition and machine learning}.
\newblock Springer, 2006.

\bibitem[Boyd and Vandenberghe(2014)]{convex}
Stephen~P. Boyd and Lieven Vandenberghe.
\newblock \emph{Convex Optimization}.
\newblock Cambridge University Press, 2014.

\bibitem[Dang and Bailey(2010)]{dang2010generation}
Xuan~Hong Dang and James Bailey.
\newblock Generation of alternative clusterings using the cami approach.
\newblock In \emph{Proceedings of the 2010 SIAM International Conference on Data Mining}, pages 118--129. SIAM, 2010.

\bibitem[Fellbaum(2010)]{fellbaum2010wordnet}
Christiane Fellbaum.
\newblock Wordnet.
\newblock In \emph{Theory and applications of ontology: computer applications}, pages 231--243. Springer, 2010.

\bibitem[Gao et~al.(2023)Gao, Geng, Zhang, Ma, Fang, Zhang, Li, and Qiao]{gao2023clip}
Peng Gao, Shijie Geng, Renrui Zhang, Teli Ma, Rongyao Fang, Yongfeng Zhang, Hongsheng Li, and Yu Qiao.
\newblock Clip-adapter: Better vision-language models with feature adapters.
\newblock \emph{International Journal of Computer Vision}, pages 1--15, 2023.

\bibitem[Geusebroek et~al.(2005)Geusebroek, Burghouts, and Smeulders]{geusebroek2005amsterdam}
Jan-Mark Geusebroek, Gertjan~J Burghouts, and Arnold~WM Smeulders.
\newblock The amsterdam library of object images.
\newblock \emph{International Journal of Computer Vision}, 61:\penalty0 103--112, 2005.

\bibitem[Gondek and Hofmann(2003)]{gondek2003conditional}
David Gondek and Thomas Hofmann.
\newblock Conditional information bottleneck clustering.
\newblock In \emph{3rd ieee international conference on data mining, workshop on clustering large data sets}, pages 36--42, 2003.

\bibitem[Gu{\'e}rin and Boots(2018)]{guerin2018improving}
Joris Gu{\'e}rin and Byron Boots.
\newblock Improving image clustering with multiple pretrained cnn feature extractors.
\newblock In \emph{British Machine Vision Conference 2018, BMVC 2018}, 2018.

\bibitem[G{\"u}nnemann et~al.(2014)G{\"u}nnemann, F{\"a}rber, R{\"u}diger, and Seidl]{gunnemann2014smvc}
Stephan G{\"u}nnemann, Ines F{\"a}rber, Matthias R{\"u}diger, and Thomas Seidl.
\newblock Smvc: semi-supervised multi-view clustering in subspace projections.
\newblock In \emph{Proceedings of the 20th ACM SIGKDD international conference on Knowledge discovery and data mining}, pages 253--262, 2014.

\bibitem[Hu and Pei(2018)]{hu2018subspace}
Juhua Hu and Jian Pei.
\newblock Subspace multi-clustering: a review.
\newblock \emph{Knowledge and information systems}, 56\penalty0 (2):\penalty0 257--284, 2018.

\bibitem[Hu et~al.(2017)Hu, Qian, Pei, Jin, and Zhu]{hu2017finding}
Juhua Hu, Qi Qian, Jian Pei, Rong Jin, and Shenghuo Zhu.
\newblock Finding multiple stable clusterings.
\newblock \emph{Knowledge and Information Systems}, 51\penalty0 (3):\penalty0 991--1021, 2017.

\bibitem[Huang et~al.(2022)Huang, Chu, and Wei]{huang2022unsupervised}
Tony Huang, Jack Chu, and Fangyun Wei.
\newblock Unsupervised prompt learning for vision-language models.
\newblock \emph{arXiv preprint arXiv:2204.03649}, 2022.

\bibitem[Krause et~al.(2013)Krause, Stark, Deng, and Fei-Fei]{krause20133d}
Jonathan Krause, Michael Stark, Jia Deng, and Li Fei-Fei.
\newblock 3d object representations for fine-grained categorization.
\newblock In \emph{Proceedings of the IEEE international conference on computer vision workshops}, pages 554--561, 2013.

\bibitem[Lloyd(1982)]{lloyd1982least}
Stuart Lloyd.
\newblock Least squares quantization in pcm.
\newblock \emph{IEEE transactions on information theory}, 28\penalty0 (2):\penalty0 129--137, 1982.

\bibitem[MacQueen et~al.(1967)]{macqueen1967some}
James MacQueen et~al.
\newblock Some methods for classification and analysis of multivariate observations.
\newblock In \emph{Proceedings of the fifth Berkeley symposium on mathematical statistics and probability}, pages 281--297. Oakland, CA, USA, 1967.

\bibitem[Menon and Vondrick(2022)]{menon2022visual}
Sachit Menon and Carl Vondrick.
\newblock Visual classification via description from large language models.
\newblock \emph{arXiv preprint arXiv:2210.07183}, 2022.

\bibitem[Miklautz et~al.(2020)Miklautz, Mautz, Altinigneli, B{\"o}hm, and Plant]{miklautz2020deep}
Lukas Miklautz, Dominik Mautz, Muzaffer~Can Altinigneli, Christian B{\"o}hm, and Claudia Plant.
\newblock Deep embedded non-redundant clustering.
\newblock In \emph{Proceedings of the AAAI conference on artificial intelligence}, pages 5174--5181, 2020.

\bibitem[Ng et~al.(2001)Ng, Jordan, and Weiss]{ng2001spectral}
Andrew Ng, Michael Jordan, and Yair Weiss.
\newblock On spectral clustering: Analysis and an algorithm.
\newblock \emph{Advances in neural information processing systems}, 14, 2001.

\bibitem[Nilsback and Zisserman(2008)]{nilsback2008automated}
Maria-Elena Nilsback and Andrew Zisserman.
\newblock Automated flower classification over a large number of classes.
\newblock In \emph{2008 Sixth Indian conference on computer vision, graphics \& image processing}, pages 722--729. IEEE, 2008.

\bibitem[Qi and Davidson(2009)]{qi2009principled}
ZiJie Qi and Ian Davidson.
\newblock A principled and flexible framework for finding alternative clusterings.
\newblock In \emph{SIGKDD}, pages 717--726, 2009.

\bibitem[Qian et~al.(2019)Qian, Shang, Sun, Hu, Tacoma, Li, and Jin]{softtriple}
Qi Qian, Lei Shang, Baigui Sun, Juhua Hu, Tacoma Tacoma, Hao Li, and Rong Jin.
\newblock Softtriple loss: Deep metric learning without triplet sampling.
\newblock In \emph{ICCV}, pages 6449--6457. {IEEE}, 2019.

\bibitem[Qian et~al.(2022)Qian, Xu, Hu, Li, and Jin]{qian2022unsupervised}
Qi Qian, Yuanhong Xu, Juhua Hu, Hao Li, and Rong Jin.
\newblock Unsupervised visual representation learning by online constrained k-means.
\newblock In \emph{Proceedings of the IEEE/CVF Conference on Computer Vision and Pattern Recognition}, pages 16640--16649, 2022.

\bibitem[Qian et~al.(2023)Qian, Xu, and Hu]{inmap}
Qi Qian, Yuanhong Xu, and Juhua Hu.
\newblock Intra-modal proxy learning for zero-shot visual categorization with clip.
\newblock In \emph{Thirty-seventh Conference on Neural Information Processing Systems, {NeurIPS} 2023}, 2023.

\bibitem[Radford et~al.(2021)Radford, Kim, Hallacy, Ramesh, Goh, Agarwal, Sastry, Askell, Mishkin, Clark, et~al.]{radford2021learning}
Alec Radford, Jong~Wook Kim, Chris Hallacy, Aditya Ramesh, Gabriel Goh, Sandhini Agarwal, Girish Sastry, Amanda Askell, Pamela Mishkin, Jack Clark, et~al.
\newblock Learning transferable visual models from natural language supervision.
\newblock In \emph{International conference on machine learning}, pages 8748--8763. PMLR, 2021.

\bibitem[Rand(1971)]{rand1971objective}
William~M Rand.
\newblock Objective criteria for the evaluation of clustering methods.
\newblock \emph{Journal of the American Statistical association}, 66\penalty0 (336):\penalty0 846--850, 1971.

\bibitem[Ren et~al.(2022)Ren, Yu, Wang, Liu, Domeniconi, and Zhang]{ren2022diversified}
Liangrui Ren, Guoxian Yu, Jun Wang, Lei Liu, Carlotta Domeniconi, and Xiangliang Zhang.
\newblock A diversified attention model for interpretable multiple clusterings.
\newblock \emph{IEEE Transactions on Knowledge and Data Engineering}, 2022.

\bibitem[Rombach et~al.(2022)Rombach, Blattmann, Lorenz, Esser, and Ommer]{sdiffusion}
Robin Rombach, Andreas Blattmann, Dominik Lorenz, Patrick Esser, and Bj{\"{o}}rn Ommer.
\newblock High-resolution image synthesis with latent diffusion models.
\newblock In \emph{CVPR}, pages 10674--10685. {IEEE}, 2022.

\bibitem[Shu et~al.(2022)Shu, Nie, Huang, Yu, Goldstein, Anandkumar, and Xiao]{shu2022test}
Manli Shu, Weili Nie, De-An Huang, Zhiding Yu, Tom Goldstein, Anima Anandkumar, and Chaowei Xiao.
\newblock Test-time prompt tuning for zero-shot generalization in vision-language models.
\newblock \emph{Advances in Neural Information Processing Systems}, 35:\penalty0 14274--14289, 2022.

\bibitem[Wang et~al.(2023)Wang, Xu, Hu, Yan, Sang, and Qian]{wang2023improved}
Junyang Wang, Yuanhong Xu, Juhua Hu, Ming Yan, Jitao Sang, and Qi Qian.
\newblock Improved visual fine-tuning with natural language supervision.
\newblock \emph{arXiv preprint arXiv:2304.01489}, 2023.

\bibitem[Wei et~al.(2020)Wei, Wang, Yu, Domeniconi, and Zhang]{wei2020multi}
Shaowei Wei, Jun Wang, Guoxian Yu, Carlotta Domeniconi, and Xiangliang Zhang.
\newblock Multi-view multiple clusterings using deep matrix factorization.
\newblock In \emph{Proceedings of the AAAI conference on artificial intelligence}, pages 6348--6355, 2020.

\bibitem[White et~al.(2004)White, Steingold, and Fournelle]{white2004performance}
JV White, Sam Steingold, and CG Fournelle.
\newblock Performance metrics for group-detection algorithms.
\newblock \emph{Proceedings of Interface}, 2004, 2004.

\bibitem[Xie et~al.(2016)Xie, Girshick, and Farhadi]{xie2016unsupervised}
Junyuan Xie, Ross Girshick, and Ali Farhadi.
\newblock Unsupervised deep embedding for clustering analysis.
\newblock In \emph{International conference on machine learning}, pages 478--487, 2016.

\bibitem[Yang and Zhang(2017)]{yang2017non}
Sen Yang and Lijun Zhang.
\newblock Non-redundant multiple clustering by nonnegative matrix factorization.
\newblock \emph{Machine Learning}, 106\penalty0 (5):\penalty0 695--712, 2017.

\bibitem[Yao et~al.(2023)Yao, Liu, Rashid, and Hu]{DBLP:conf/inns-dlia/YaoLRH23}
Jiawei Yao, Enbei Liu, Maham Rashid, and Juhua Hu.
\newblock Augdmc: Data augmentation guided deep multiple clustering.
\newblock In \emph{{INNS} DLIA@IJCNN}, 2023.

\bibitem[Zhou et~al.(2022)Zhou, Yang, Loy, and Liu]{zhou2022learning}
Kaiyang Zhou, Jingkang Yang, Chen~Change Loy, and Ziwei Liu.
\newblock Learning to prompt for vision-language models.
\newblock \emph{International Journal of Computer Vision}, 130\penalty0 (9):\penalty0 2337--2348, 2022.

\end{thebibliography}
}

\clearpage
\setcounter{page}{1}
\appendix
\maketitlesupplementary

\section{Additional Results}
\subsection{Clustering Analysis}

We further analyze the clustering outcomes of \sysname{} against various clusters. The results are derived by applying different embeddings to produce all clustering outputs. Then, we compare each obtained clustering result to all ground truth clusterings. The findings, as depicted in Tables~\ref{tab:clustering1} and~\ref{tab:clustering2}, reveal that the top-performing outcomes exhibit a clear diagonal structure. This demonstrates that the representations generated by \sysname{} are capable of discerning different aspects of the same data, and then produce different clusterings aligning well with different ground truth data structures.

\begin{table}[h]
    \centering
    \resizebox{0.9\columnwidth}{!}{
    \begin{tabular}{cc|cccc}
    \toprule
        \multirow{2}{*}{Dataset} & \multirow{2}{*}{Clustering} & \multicolumn{2}{c}{$C
_1$} & \multicolumn{2}{c}{$C_2$} \\ 
        ~ & ~ & NMI & RI & NMI & RI \\ 
        \midrule
        \multirow{2}{*}{ALOI~\cite{geusebroek2005amsterdam}} & Color & \textcolor{blue}{\textbf{1.0000}} & \textcolor{blue}{\textbf{1.0000}} & 0.3164 & 0.6798 \\ 
        ~ & Shape & 0.3642 & 0.5491 & \textcolor{blue}{\textbf{1.0000}} & \textcolor{blue}{\textbf{1.0000}} \\ 
        \midrule
        \multirow{2}{*}{Fruit~\cite{hu2017finding}} & Color & \textcolor{blue}{\textbf{0.8619}} & \textcolor{blue}{\textbf{0.9526}} & 0.5379 & 0.6934 \\ 
        ~ & Species & 0.6953 & 0.7843 & \textcolor{blue}{\textbf{1.0000}} & \textcolor{blue}{\textbf{1.0000}} \\ 
        \midrule
        \multirow{2}{*}{Fruit360~\cite{DBLP:conf/inns-dlia/YaoLRH23}} & Color & \textcolor{blue}{\textbf{0.6239}} & \textcolor{blue}{\textbf{0.8243}} & 0.4242 & 0.7163 \\ 
        ~ & Species & 0.3551 & 0.6333 & \textcolor{blue}{\textbf{0.5284}} & \textcolor{blue}{\textbf{0.7582}} \\ 
        \midrule
        \multirow{2}{*}{Card~\cite{DBLP:conf/inns-dlia/YaoLRH23}} & Order & \textcolor{blue}{\textbf{0.3653}} & \textcolor{blue}{\textbf{0.8587}} & 0.1142 & 0.5562 \\ 
        ~ & Suits & 0.1346 & 0.5679 & \textcolor{blue}{\textbf{0.2734}} & \textcolor{blue}{\textbf{0.7039}} \\ 
        \midrule
        \multirow{2}{*}{Stanford Cars~\cite{krause20133d}} & Color & \textcolor{blue}{\textbf{0.7360}} & \textcolor{blue}{\textbf{0.9193}} & 0.4223 & 0.7526 \\ 
        ~ & Type & 0.3692 & 0.6245 & \textcolor{blue}{\textbf{0.6355}} & \textcolor{blue}{\textbf{0.8399}} \\ 
        \midrule
        \multirow{2}{*}{Flowers~\cite{nilsback2008automated}} & Color & \textcolor{blue}{\textbf{0.6426}} & \textcolor{blue}{\textbf{0.7984}} & 0.3277 & 0.7153 \\ 
        ~ & Species & 0.2884 & 0.6150 & \textcolor{blue}{\textbf{0.6013}} & \textcolor{blue}{\textbf{0.8103}} \\ 
        \bottomrule
    \end{tabular}
    }
    \caption{Clustering analysis in six benchmark multiple clustering vision tasks.}
    \label{tab:clustering1}
\end{table}

\begin{table}[h]
    \centering
    \resizebox{0.4\textwidth}{!}{
    \begin{tabular}{cc|cccc}
    \toprule
        Clustering & Metrics & C1 & C2 & C3 & C4 \\ 
        \midrule
        \multirow{2}{*}{Emotion} & NMI & \textcolor{blue}{\textbf{0.1786}} & 0.0362 & 0.0564 & 0.0435 \\ 
        ~ & RI & \textcolor{blue}{\textbf{0.7105}} & 0.4376 & 0.513 & 0.5683 \\ 
        \midrule
        \multirow{2}{*}{Glass} & NMI & 0.1104 & \textcolor{blue}{\textbf{0.3402}} & 0.1163 & 0.1567 \\ 
        ~ & RI & 0.6893 & \textcolor{blue}{\textbf{0.7068}} & 0.6429 & 0.6952 \\ 
        \midrule
        \multirow{2}{*}{Identity} & NMI & 0.2342 & 0.3627 & \textcolor{blue}{\textbf{0.6625}} & 0.325 \\ 
        ~ & RI & 0.5362 & 0.7632 & \textcolor{blue}{\textbf{0.9496}} & 0.7117 \\ 
        \midrule
        \multirow{2}{*}{Pose} & NMI & 0.0673 & 0.1258 & 0.1368 & \textcolor{blue}{\textbf{0.4693}} \\ 
        ~ & RI & 0.4989 & 0.5519 & 0.5867 & \textcolor{blue}{\textbf{0.6624}} \\ 
        \bottomrule
    \end{tabular}
    }
    \caption{Clustering analysis on CMUface~\cite{gunnemann2014smvc} datasets.}
    \label{tab:clustering2}
\end{table}

\subsection{Efficiency Analysis}

To further demonstrate the efficiency of our proposed method using the frozen pre-trained model by CLIP, we compare the efficiency of different deep multiple clustering methods. The experiments are conducted on a server with a GPU GeForece RTX 2080Ti. We show the running time on Fruit dataset. The running time and color clustering performance of each method are shown in Fig.~\ref{fig:time}. \sysname{} has significantly better performance than all the baselines in both effectiveness and efficiency. 
That is because our method can directly exploit the CLIP encoder to capture the image and text embeddings, without updating the encoder’s parameters, so its running time is much smaller than other methods. In summary, the proposed method shows the best performance under the least running time requirement.

\begin{figure}[h]
    \centering
    \includegraphics[width=0.8\columnwidth]{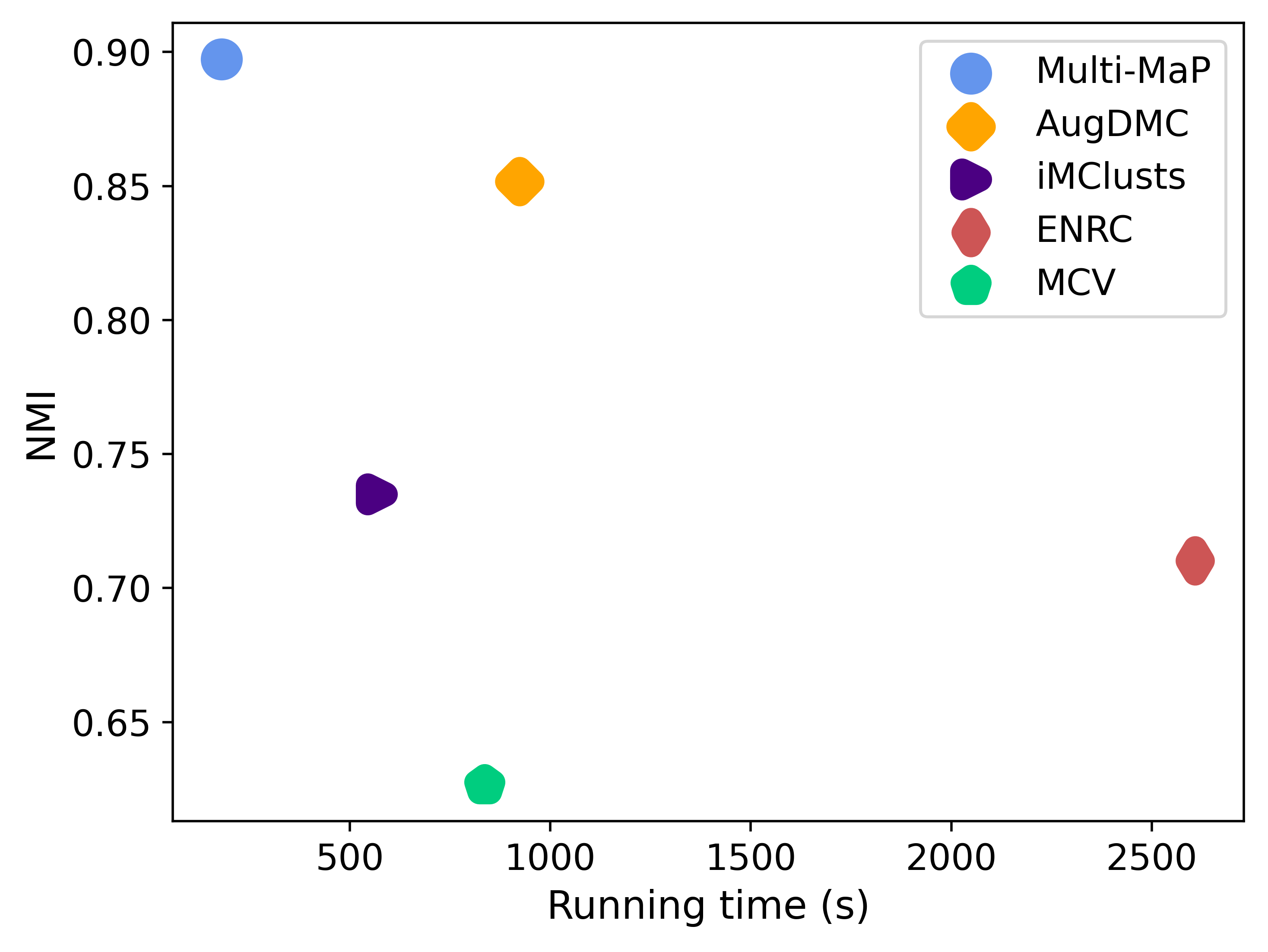}
    \caption{Performance vs. the running time on Fruit~\cite{hu2017finding} dataset.}
    \label{fig:time}
\end{figure}

\begin{table}[h]
    \centering
    \resizebox{\columnwidth}{!}{
    \begin{tabular}{cc|cccccccccc}
    \toprule

        \multirow{2}{*}{Dataset} & \multirow{2}{*}{Clustering} & \multicolumn{2}{c}{\sysname{}$_\text{wo{CR}}$} & \multicolumn{2}{c}{\sysname{}$_\text{woC}$} & \multicolumn{2}{c}{\sysname{}$_\text{woR}$} & \multicolumn{2}{c}{\sysname{}$_\text{MSE}$} & \multicolumn{2}{c}{\sysname{}}\\
        
        ~ & ~ & NMI & RI & NMI & RI & NMI & RI & NMI & RI & NMI & RI \\ 
        \midrule
        \multirow{2}{*}{ALOI~\cite{geusebroek2005amsterdam}} & Color & 0.9632 & 0.9829 & 1.0000 & 1.0000 & 0.9843 & 0.9906 & 1.0000 & 1.0000 & 1.0000 & 1.0000 \\ 
        ~ & Shape & 1.0000 & 1.0000 & 1.0000 & 1.0000 & 1.0000& 1.0000 & 1.0000 & 1.0000 & 1.0000 & 1.0000 \\ 
        \midrule
        \multirow{2}{*}{Fruit~\cite{hu2017finding}} & Color & 0.7634 & 0.8432 & 0.8212 & 0.9274 & 0.8169 & 0.9198 & 0.8479 & 0.9296 & \textcolor{blue}{\textbf{0.8619}} & \textcolor{blue}{\textbf{0.9526}} \\ 
        ~ & Species & 1.0000 & 1.0000 & 1.0000 & 1.0000 & 1.0000 & 1.0000 & 1.0000 & 1.0000 & 1.0000 & 1.0000 \\ 
        \midrule
        \multirow{2}{*}{Fruit360~\cite{DBLP:conf/inns-dlia/YaoLRH23}} & Color & 0.5634 & 0.7650 & 0.6209 & 0.7825 & 0.6134 & 0.8036 & 0.6082 & 0.7943 & \textcolor{blue}{\textbf{0.6239}} & \textcolor{blue}{\textbf{0.8243}} \\ 
        ~ & Species & 0.5077 & 0.7368 & 0.5137 & 0.7436 & 0.5176 & 0.7363 & 0.5199 & 0.7428 & \textcolor{blue}{\textbf{0.5284}} & \textcolor{blue}{\textbf{0.7582}} \\ 
        \midrule
        \multirow{2}{*}{Card~\cite{DBLP:conf/inns-dlia/YaoLRH23}} & Order & 0.1928 & 0.8136 & 0.3560 & 0.8458 & 0.3518 & 0.8458 & 0.3605 & 0.8509 & \textcolor{blue}{\textbf{0.3653}} & \textcolor{blue}{\textbf{0.8587}} \\ 
        ~ & Suits & 0.2374 & 0.6271 & 0.2691 & 0.6632 & 0.2481 & 0.6104 & 0.2550 & 0.6596 & \textcolor{blue}{\textbf{0.2734}} & \textcolor{blue}{\textbf{0.7039}} \\ 
        \midrule
        \multirow{4}{*}{CMUface~\cite{gunnemann2014smvc}} & Emotion & 0.1692 & 0.6169 & 0.1717 & 0.6233 & 0.1709 & 0.6662 & 0.1711 & 0.6843 & \textcolor{blue}{\textbf{0.1786}} & \textcolor{blue}{\textbf{0.7105}} \\ 
        ~ & Glass & 0.3107 & 0.6902 & 0.3265 & 0.7130 & 0.3194 & 0.6908 &0.3362 & 0.7039 & \textcolor{blue}{\textbf{0.3402}} & \textcolor{blue}{\textbf{0.7068}} \\ 
        ~ & Identity & 0.5632 & 0.8236 & 0.6236 & 0.8368 & 0.6042 & 0.8273 &0.6396 & 0.8941 & \textcolor{blue}{\textbf{0.6625}} & \textcolor{blue}{\textbf{0.9496}} \\ 
        ~ & Pose & 0.4361 & 0.6407 & 0.4556 & 0.6492 & 0.4405 & 0.6507 &0.4398 & 0.6479 & \textcolor{blue}{\textbf{0.4693}} & \textcolor{blue}{\textbf{0.6624}} \\ 
        \midrule
        \multirow{2}{*}{Stanford Cars~\cite{krause20133d}} & Color & 0.5933 & 0.7832 & 0.6834 & 0.8665 & 0.6942 & 0.8930 & 0.7114 & 0.9105 & \textcolor{blue}{\textbf{0.7360}} & \textcolor{blue}{\textbf{0.9193}} \\ 
        ~ & Brand & 0.5562 & 0.7993 & 0.6388 & 0.8263 & 0.6207 & 0.7931 &0.6287 & 0.8176 & \textcolor{blue}{\textbf{0.6355}} & \textcolor{blue}{\textbf{0.8399}} \\ 
        \midrule
        \multirow{2}{*}{Flowers~\cite{nilsback2008automated}} & Color & 0.5795 & 0.7719 & 0.5836 & 0.7838 & 0.6133 & 0.7990 & 0.6211 & 0.7936 & \textcolor{blue}{\textbf{0.6426}} & \textcolor{blue}{\textbf{0.7984}} \\ 
        ~ & Species & 0.5699 & 0.7604 & 0.5737 & 0.7836 & 0.5905 & 0.8012 & 0.5842 & 0.7897 & \textcolor{blue}{\textbf{0.6013}} & \textcolor{blue}{\textbf{0.8103}} \\ 
        \bottomrule
    \end{tabular}
    }
     \caption{Components in Multi-MaP. The significantly best results with 95\% confidence are in bold.}
    \label{tab:ablation}
\end{table}

\subsection{Ablation Study}

To validate the effectiveness of \sysname{}, we compare four variants of \sysname{} that are removing the reference word constraint, removing the concept-level constraint, removing both constraints and implementing reference constraint with a high-level concept provided by a user, denoted as \sysname{}$_\text{woR}$, \sysname{}$_\text{woC}$, \sysname{}$_\text{woCR}$ and \sysname{}$_\text{MSE}$, respectively. 
The results are shown in Table~\ref{tab:ablation}. The proposed method achieved the best results, while the method that removed both reference word constraint and concept-level constraint performed the worst. This also shows that the proposed reference word constraint and concept-level constraint play an important role in the model. Moreover, \sysname{} performs better than \sysname{}$_\text{MSE}$, suggesting Multi-MaP can benefit from multiple concepts through the contrastive learning process.

\end{document}